\documentclass[sigconf]{acmart}

\copyrightyear{2025}
\acmYear{2025}
\setcopyright{cc}
\setcctype{by}
\acmConference[KDD '25]{Proceedings of the 31st ACM SIGKDD Conference on Knowledge Discovery and Data Mining V.2}{August 3--7, 2025}{Toronto, ON, Canada}
\acmBooktitle{Proceedings of the 31st ACM SIGKDD Conference on Knowledge Discovery and Data Mining V.2 (KDD '25), August 3--7, 2025, Toronto, ON, Canada}
\acmDOI{10.1145/3711896.3736868}
\acmISBN{979-8-4007-1454-2/2025/08}

\usepackage{amsmath,amsfonts,bm,dsfont,amsthm,numprint}
\usepackage{algorithm2e}
\usepackage{graphicx}
\usepackage{float}
\usepackage{subcaption}
\usepackage{listings}
\definecolor{codegreen}{rgb}{0,0.6,0}
\definecolor{codegray}{rgb}{0.5,0.5,0.5}
\definecolor{codepurple}{rgb}{0.58,0,0.82}

\lstdefinestyle{mystyle}{
    backgroundcolor=\color{white},   
    commentstyle=\color{codegreen},
    keywordstyle=\color{magenta},
    stringstyle=\color{codepurple},
    escapeinside={(*@}{@*)},
    basicstyle=\ttfamily\small,
    breakatwhitespace=false,         
    breaklines=false,                 
    captionpos=b,                    
    keepspaces=true,                 
    numbers=none,                    
    numbersep=3pt,                  
    showspaces=false,                
    showstringspaces=false,
    showtabs=false,                  
    tabsize=2
}

\DeclareMathOperator*{\argmax}{arg\,max}

\begin{document}

\title{Breiman meets Bellman: Non-Greedy Decision Trees with MDPs}

\author{Hector Kohler}
\email{hector.kohler@inria.fr}
\affiliation{%
\institution{UMR 9189 - CRIStAL}
  \institution{Universit\'e de Lille}
  \institution{CNRS, Inria, Centrale Lille}
  \city{Lille}
  \country{France}
}

\author{Riad Akrour}
\email{riad.akrour@inria.fr}
\affiliation{%
\institution{UMR 9189 - CRIStAL}
  \institution{Universit\'e de Lille}
  \institution{Inria, CNRS, Centrale Lille}
  \city{Lille}
  \country{France}
}

\author{Philippe Preux}
\email{philippe.preux@inria.fr}
\affiliation{%
\institution{UMR 9189 - CRIStAL}
  \institution{Universit\'e de Lille}
  \institution{CNRS, Inria, Centrale Lille}
  \city{Lille}
  \country{France}
}


\begin{abstract}
In supervised learning, decision trees are valued for their interpretability and performance. 
While greedy decision tree algorithms like CART remain widely used due to their computational efficiency, they often produce sub-optimal solutions with respect to a regularized training loss. 
Conversely, optimal decision tree methods can find better solutions but are computationally intensive and typically limited to shallow trees or binary features. We present Dynamic Programming Decision Trees (DPDT), a framework that bridges the gap between greedy and optimal approaches. 
DPDT relies on a Markov Decision Process formulation combined with heuristic split generation to construct near-optimal decision trees with significantly reduced computational complexity. 
Our approach dynamically limits the set of admissible splits at each node while directly optimizing the tree regularized training loss. Theoretical analysis demonstrates that DPDT can minimize regularized training losses at least as well as CART\@. 
Our empirical study shows on multiple datasets that DPDT achieves near-optimal loss with orders of magnitude fewer operations than existing optimal solvers. 
More importantly, extensive benchmarking suggests statistically significant improvements of DPDT over both CART and optimal decision trees in terms of generalization to unseen data. We demonstrate DPDT practicality through applications to boosting, where it consistently outperforms baselines. 
Our framework provides a promising direction for developing efficient, near-optimal decision tree algorithms that scale to practical applications.
\end{abstract}


\begin{teaserfigure}
  \includegraphics[width=\textwidth]{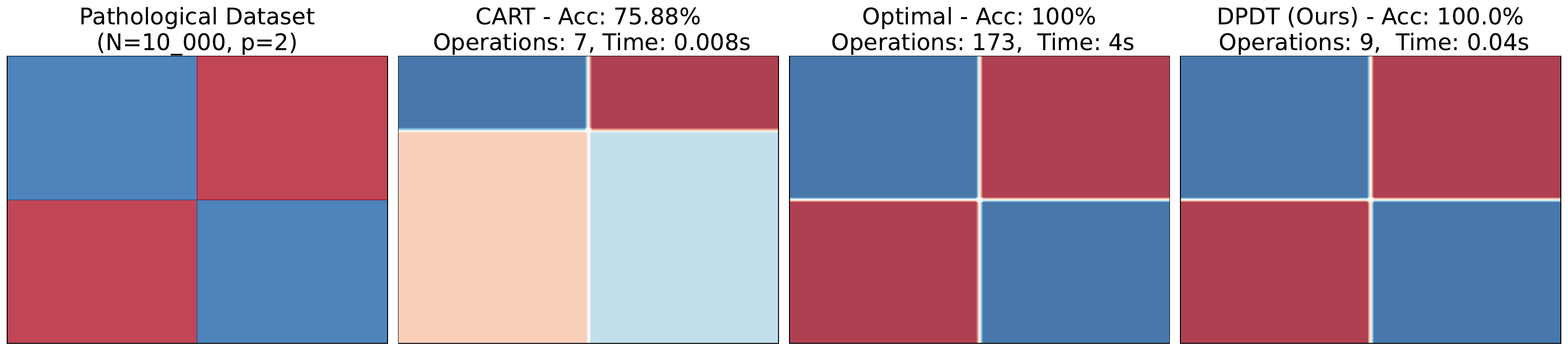}
  \caption{Pathological dataset and learned depth-2 trees with their scores, complexities, runtimes, and decision boundaries.}
  \Description{}
  \label{fig:patho}
\end{teaserfigure}

\maketitle

\section{Introduction}
Decision trees~\cite{ID3,c45,breiman1984classification} are at the core of various machine learning applications. 
Ensembles of decision trees such as tree boosting~\cite{stcohFriedman,FriedmanBoosting,xgb,10.5555/3327757.3327770} are the state-of-the-art for supervised learning on tabular data \citep{grinsztajn2022tree}.
Human users can make sense of decision trees predictions~\cite{rigourous,lipton,pmlr-v247-bressan24a} which allows for real-world applications when safety or trust is critical \citep{saux:hal-04192198}. 
More recently, decision trees have been used to model sequential decision policies with imitation learning~\cite{viper,kohler2024interpretable} or directly with reinforcement learning (RL) \citep{topin2021iterative,vos2024optimizinginterpretabledecisiontree,mdpdt,marton2024sympolsymbolictreebasedonpolicy}.

To motivate the design of new decision tree induction algorithms, Figure~\ref{fig:patho} exhibits a dataset for which existing greedy algorithms are suboptimal, and optimal algorithms are computationally expensive. 
The dataset is made up of $N=10^4$ samples in $p=2$ dimensions that can be perfectly labeled with a decision tree of depth 2. When running CART \citep{breiman1984classification}, greedily choosing the root node yields a suboptimal tree.
This is because greedy algorithms compute locally optimal splits in terms of information gain. In our example, the greedy splits always give two children datasets which themselves need depth 2 trees to be perfectly split.
On the other hand, to find the root node, an optimal algorithm such as \citep{quantbnb} iterates over all possible splits, that is, $N\times p=\numprint{20,000}$ operations to find one node of the solution tree.

In this work, we present a framework for designing non-greedy decision tree induction algorithms that optimize a regularized training loss nearly as well as optimal methods. This is achieved with orders of magnitude less operations, and hence dramatic computation savings.
We call this framework ``Dynamic Programming Decision Trees'' (DPDT). For every node, DPDT heuristically and dynamically limits the set of admissible splits to a few good candidates. Then, DPDT optimizes the regularized training loss with some depth constraints.
Theoretically, we show that DPDT minimizes the empirical risk at least as well as CART\@.
Empirically, we show that on all tested datasets, DPDT can reach 99\% of the optimal regularized train accuracy while using thousands times less operations than current optimal solvers. 
More importantly, we follow \citep{grinsztajn2022tree} methodology to benchmark DPDT against both CART and optimal trees on hard datasets. Following the same methodology, we compare boosted DPDT \citep{FREUND1997119} to boosted CART and to some deep learning methods and show clear superiority of DPDT.

\section{Related Work}

To learn decision trees, greedy approaches like CART \citep{breiman1984classification} iteratively partition the training dataset by taking splits optimizing a local objective such as the Gini impurity or the entropy. 
This makes CART suboptimal with respect to training losses \citep{Murthy}. 
But CART remains the default decision tree algorithm in many machine learning libraries such as \citep{scikit-learn,xgb,ke2017lightgbm,9533597} because it can scale to very deep trees and is very fast.
To avoid overfitting, greedy trees are learned with a maximal depth or pruned a posteriori \citep[chapter~3]{breiman1984classification}. 
In recent years, more complex optimal decision tree induction algorithms have shown consistent gains over CART in terms of generalization capabilities \citep{oct,verwer2017learning,murtree}.

Optimal decision tree approaches optimize a regularized training loss while using a minimal number of splits~\cite{oct,mfoct,binoct,quantbnb,murtree,blossom,pystreed,chaouki2024branchesfastdynamicprogramming}.
However, direct optimization is not a convenient approach, as finding the optimal tree is known to be NP-Hard \citep{npcomplete}. Despite the large number of algorithmic tricks to make optimal decision tree solvers efficient~\cite{murtree,quantbnb}, their complexity scales with the number of samples and the maximum depth constraint.
Furthermore, optimal decision tree induction algorithms are usually constrained to binary-features dataset while CART can deal with any type of feature. When optimal decision tree algorithms deal with continuous features, they can usually learn only shallow trees, e.g. Quant-BnB \citep{quantbnb} can only compute optimal trees up to depth 3.
\texttt{PySTreeD}, the latest optimal decision tree library~\cite{pystreed}, can compute decision trees with depths larger than three but uses heuristics to binarize a dataset with continuous features during a pre-processing step. 
Despite their limitations to binary features and their huge computational complexities, encouraging practical results for optimal trees have been obtained \cite{how-eff,lin2020generalized,costa2023recent,vanderlinden2024optimalgreedydecisiontrees}.
Among others, they show that optimal methods under the same depth constraint (up to depth four) find
trees with 1--2\% greater test accuracy than greedy methods.

In this work, we only consider the induction of nonparametric binary depth-constrained axis-aligned trees. By nonparametric trees, we mean that we only consider tree induction algorithms that optimize both features and threshold values in internal nodes of the tree. This is different from the line of work on Tree Alternating Optimization (TAO) algorithm~\cite{NEURIPS2018_185c29dc,9534446,10.1145/3412815.3416882} that only optimizes tree nodes threshold values for fixed nodes features similarly to optimizing neural network weights with gradient-based methods. 

There exist many other areas of decision tree research \citep{loh2014fifty} such as inducing non-axis parallel decision trees \citep{murthy1994system,10.1145/3637528.3671903}, splitting criteria of greedy trees \citep{vanderlinden2024optimalgreedydecisiontrees}, different optimization of parametric trees \citep{NIPS2015_1579779b,10.5555/3327757.3327770}, or pruning methods \citep{pruning1,pruning2}. 

Our work is not the first to formulate the decision tree induction problem as solving a Markov decision process (MDP)~\cite{Dulac_Arnold_2011,garlapati2015reinforcementlearningapproachonline,topin2021iterative,chaouki2024branchesfastdynamicprogramming}. Those works formulate tree induction as solving a partially observable MDP and use approximate algorithms such as Q-learning  \citep{garlapati2015reinforcementlearningapproachonline} or deep Q-learning \citep{topin2021iterative} to solve them in an online fashion one datum from the dataset at a time. In a nutshell, our work, DPDT that we present next, is different in that it builds a stochastic and fully observable MDP that can be explicitly solved with dynamic programming. This makes it possible to solve exactly the decision tree induction problem. 

\section{Decision Trees for Supervised Learning}


Let us briefly introduce some notations for the supervised classification problem considered in this paper.
We assume that we have access to a set of $N$ examples denoted $\mathcal{E} = {\{(x_i, y_i)\}}_{i=1}^N$. Each datum $x_i$ is described by a set of $p$ features. $y_i \in {\mathcal Y}$ is the label associated with $x_i$.

A decision tree is made of two types of nodes: split nodes that are traversed, and leaf nodes that finally assign a label.
To predict the label of a datum $x$, a decision tree $T$ sequentially applies a series of splits before assigning it a label $T(x) \in \mathcal Y$.
In this paper, we focus on binary decision trees with axis-aligned splits as in~\cite{breiman1984classification}, where each split compares the value of one feature with a threshold. 

Our goal is to learn a tree that generalizes well to unseen data. To avoid overfitting, we constrain the maximum depth $D$ of the tree, where $D$ is the maximum number of splits that can be applied to classify a data. We let $\mathcal{T}_{D}$ be the set of all binary decision trees of depth $\leq D$. Given a loss function ${\ell}: \mathcal{Y} \times \mathcal{Y} \rightarrow \mathbb{R}_+$, we induce trees with a regularized training loss defined by:

\begin{align}
    T^* &= \underset{T \in \mathcal{T}_D}{\operatorname{argmin}}\ {\mathcal L}_\alpha(T), \nonumber \\
    T^* &= \underset{T \in \mathcal{T}_D}{\operatorname{argmin}}\ \frac{1}{N}\overset{N}{\underset{i=1}{\sum}}{\ell}(y_i, T(x_i)) + \alpha C(T),
    \label{eq:suplearning}
\end{align}

where $C: \mathcal{T} \rightarrow \mathbb{R}$ is a complexity penalty that helps prevent or reduce overfitting such as the number of nodes~\cite{breiman1984classification,quantbnb}, or the expected number of splits to label a data\citep{how-eff}. The complexity penalty is weighted by $\alpha \in [0, 1]$. 
For supervised classification problems, we use the 0--1 loss: $\ell(y_i, T(x_i)) = \mathds{1}_{\{y_i\neq T(x_i)\}}$. Please note while we focus on supervised classification problems in this paper, our framework extends naturally to regression problems.

We now formulate the decision tree induction problem \ref{eq:suplearning} as finding the optimal policy in an MDP.

\section{Decision Tree Induction as an MDP}\label{sec:MDP}


Given a set of examples $\mathcal{E}$, the induction of a decision tree is made of a sequence of decisions: at each node, we must decide whether it is better to split (a subset of) $\mathcal{E}$, or 
to create a leaf node.

This sequential decision-making process corresponds to a Markov Decision Problem (MDP) \citep{puterman} $\mathcal{M}=\langle S, A, R_{\alpha}, P, D \rangle$.
A state is a pair made of a subset of examples $X\subseteq\mathcal E$ and a depth $d$. Then, the set of states is $S = \{ (X, d) \in P(\mathcal{E}) \times \{0, \ldots, D\} \}$ where $P(\mathcal{E})$ denotes the power set of $\mathcal{E}$. $d \in \{0,\ldots,D\}$ is the current depth in the tree.
An action $A$ consists in creating either a split node, or a leaf node (label assignment). We denote the set of candidate split nodes 
$ {\mathcal F} $. A split node in $\mathcal F$ is a pair made of one feature $i$ and a threshold value $x_{ij}\in \mathcal{E}$.
So, we can write $A = {\mathcal{F} \cup \{ 1, \ldots, K \}}$.
From state $s=(X,d)$ and a splitting action $a \in {\mathcal F}$, the transition function $P$ moves to the next state $s_l = (X_l, d+1)$ with probability $p_l = \frac{|X_l|}{|X|}$ where $X_l = \{(x_i, y_i) \in X: x_i \leq x_{ij}\}$, or to state $s_r = (X \setminus X_l, d+1)$ with probability $1-p_l$. For a class assignment action $a \in \{1,\ldots,K\}$, the chain reaches an absorbing terminal state with probability 1. 
The reward function $R_{\alpha}: S \times A \rightarrow \mathbb{R}$ returns $-\alpha$ for splitting actions and the proportion of misclassified examples of $X$ $-\frac{1}{|X|}\sum_{(x_i,y_i) \in X} \ell(y_i, a)$ for class assignment actions. $\alpha \in [0,1]$ controls the accuracy-complexity trade-off defined in the regularized training objective \ref{eq:suplearning}. 
The horizon $D$ limits tree depth to $D$ by forbidding class assignments after $D$ MDP transitions.

The solution to this MDP is a deterministic policy $\pi: S \rightarrow A$ that maximizes $J_{\alpha}(\pi) ={\mathbb{E}}\left[\sum_{t = 0}^D R_{\alpha}(s_t, \pi(s_t))\right]$, the expected sum of rewards where the expectation is taken over transitions $s_{t+1}\sim P(s_t, \pi(s_t))$ starting from initial state $s_0 = (\mathcal{E}, 0)$. 
Any such policy can be converted into a binary decision tree through a recursive extraction function $E(\pi, s)$ that returns, either a leaf node with class $\pi(s)$ if $\pi(s)$ is a class assignment, or a tree with root node containing split $\pi(s)$ and left/right sub-trees $E(\pi, s_l)$/$E(\pi, s_r)$ if $\pi(s)$ is a split. The final decision tree $T$ is obtained by calling $E(\pi, s_0)$ on the initial state $s_0$. 

\begin{proposition}[Objective Equivalence]\label{prop:equiv}
Let $\pi$ be a deterministic policy of the MDP and $\pi^*$ be an optimal deterministic policy. 
Then $J_\alpha(\pi) = -{\mathcal L}_\alpha(E(\pi, s_0))$ and $T^* = E(\pi^*, s_0)$ where $T^*$ is a tree that optimizes Eq.~\ref{eq:suplearning}.
\end{proposition}

This proposition is key as it states that the return of any policy of the MDP defined above is equal to the regularized training accuracy of the tree extracted from this policy. A consequence of this proposition is that when all possible splits are considered, the optimal policy will generate the optimal tree in the sense defined by Eq.~\eqref{eq:suplearning}. The proof is given in the Appendix \ref{sec:proof-equiv}.

\section{Algorithm}\label{sec:dpdt}

We now present the Dynamic Programming Decision Tree (DPDT) induction algorithm. 
The algorithm consists of two essential steps. The first and most computationally expensive step constructs the MDP presented in Section~\ref{sec:MDP}. 
The second step solves it to obtain a policy that maximizes Eq.\ref{sec:MDP} and that is equivalent to a decision tree. Both steps are now detailed.

\subsection{Constructing the MDP}

An algorithm constructing the MDP of section~\ref{sec:MDP} essentially computes the set of all possible decision trees of maximum depth $D$ which decision nodes are in $\mathcal F$. 
The transition function of this specific MDP is a directed acyclic graph. Each node of this graph corresponds to a state for which one computes the transition and reward functions. 
Considering all possible splits in $\mathcal F$ does not scale. We thus introduce a state-dependent action space $A_s$, much smaller than $A$ and populated by a splits generating function. In Figure \ref{fig:schema-mdp}, we illustrate the MDP constructed for the classification of a toy dataset using some arbitrary splitting function.

\begin{figure}
      \centering
      \includegraphics[width=0.9\linewidth]{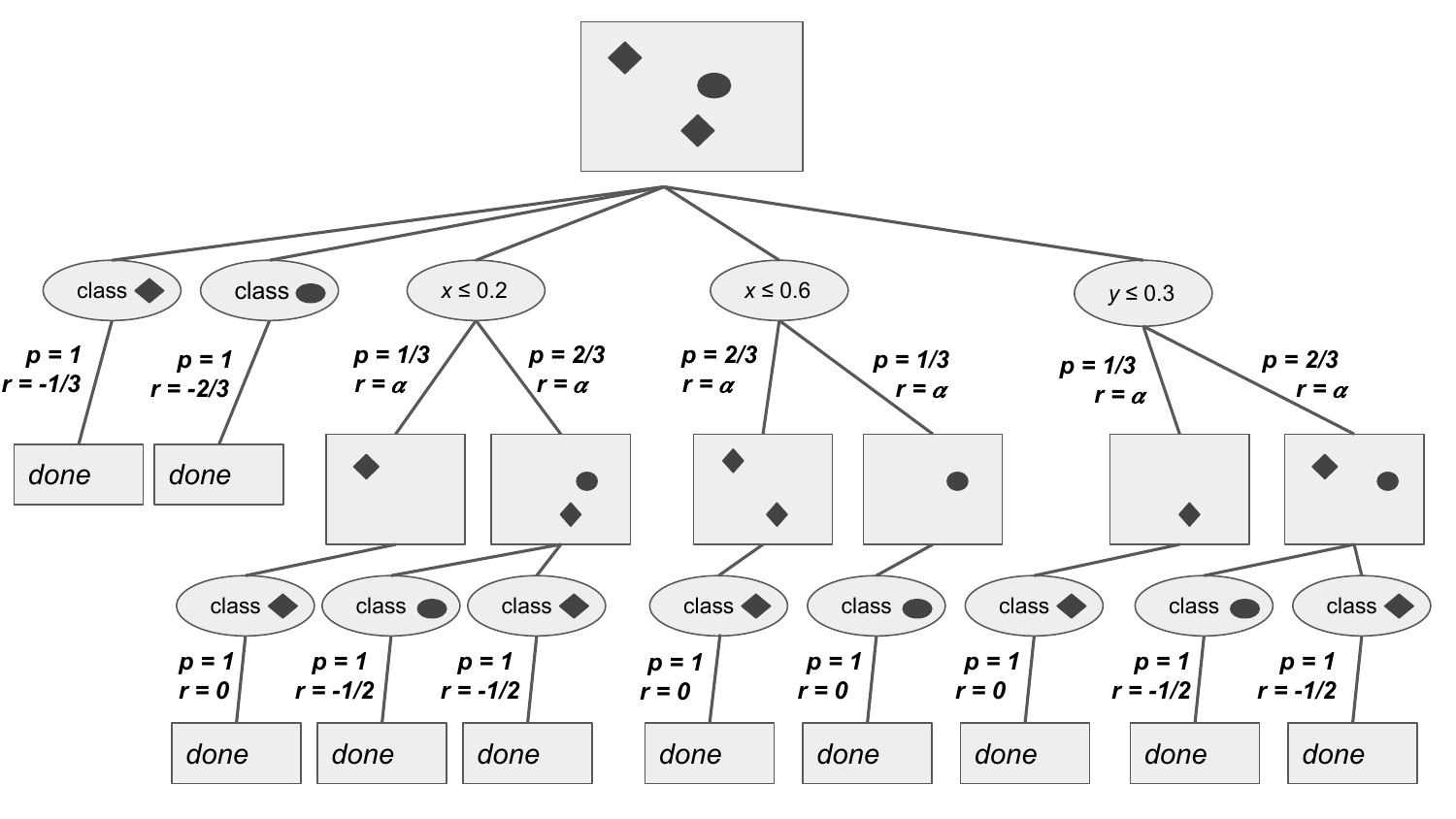}
      \caption{Schematics of the MDP to learn a decision tree of depth 2 to classify a toy dataset with three samples, two features (x,y), and two classes (oval, diamond) and using an arbitrary splits generating function.}
      \Description{This figure represent 
a schematic of the MDP described in previous sections. It looks like an acyclic graph where nodes are either MDP actions or subsets of training examples.}\label{fig:schema-mdp}
\end{figure}

\subsection{Heuristic splits generating functions}\label{sec:testgen}

A split generating function is any function $\phi$ that maps an MDP state, i.e., a subset of training examples, to a split node. It has the form $\phi: S \rightarrow P(\mathcal{F})$, where $P(\mathcal{F})$ is the power set of all possible split nodes in $\mathcal F$. 
For a state $s \in S$, the state-dependent action space is defined by $A_s = \phi(s) \cup  \{1,\ldots,K\}$. 

When the split generating function does not return all the possible candidate split nodes given a state, solving the MDP with state-dependent actions $A_s$ is not guaranteed to yield the minimizing tree of Eq.~\ref{eq:suplearning}, as the 
optimization is then performed on the subset of trees of depth smaller or equal to $D$, $\mathcal{T}_D$. 
We now define some interesting split generating functions and provide the time complexity of the associated decision tree algorithms. The time complexity is given in big-O of the number of candidate split nodes considered during computations. 

\paragraph{Exhaustive function.} When $\mathcal{F} \subseteq \phi(s), \forall s \in S$, the MDP contains all possible splits of a certain set of examples. In this case, \textit{the optimal MDP policy is the optimal decision tree of depth at most D},
and the number of states of the MDP would be $O({(2Np)}^D)$. Solving the MDP for $A_s = \phi(s)$ is equivalent to running one of the optimal tree induction algorithms~\cite{verwer2017learning,oct,pystreed,quantbnb,binoct,murtree,mfoct,blossom,lin2020generalized,chaouki2024branchesfastdynamicprogramming}

\paragraph{Top $B$ most informative splits.}\label{topk-heuristic}~\cite{topk} proposed to generate splits with a function that returns, for any state $s=(X,d)$, the $B$ most informative splits over $X$ with respect to some information gain measure such as the entropy or the Gini impurity. 
The number of states in the MDP would be $O({(2B)}^D)$. \textit{When $B=1$, the optimal policy of the MDP is the greedy tree.} 
In practice, we noticed that the returned set of splits lacked diversity and often consists of splits on the same feature with minor changes to the threshold value. 

\paragraph{Calls to CART}\label{cart-heuristic} Instead of returning the most informative split at each state $s=(X,d)$, we propose to find the most discriminative split, i.e.\@ the feature splits that best predicts the class of data in $X$.
We can do this by considering the split nodes of the greedy tree. In practice, we run CART on $s$ and use the returned nodes as $\phi(s)$. We control the number of MDP states by constraining CART trees with a maximum number of nodes $B$: $\phi(s) = nodes(\text{CART}(s, max\_nodes=B))$
The number of MDP states would be $O({(2B)}^D)$.\textit{When $B=1$, the MDP policy corresponds to the greedy tree.} The process of generating split nodes with calls to CART is illustrated in Figure \ref{fig:schema-dpdt}.

\begin{figure}
      \centering
      \includegraphics[trim={0 0cm 0 0},clip,width=0.5\textwidth]{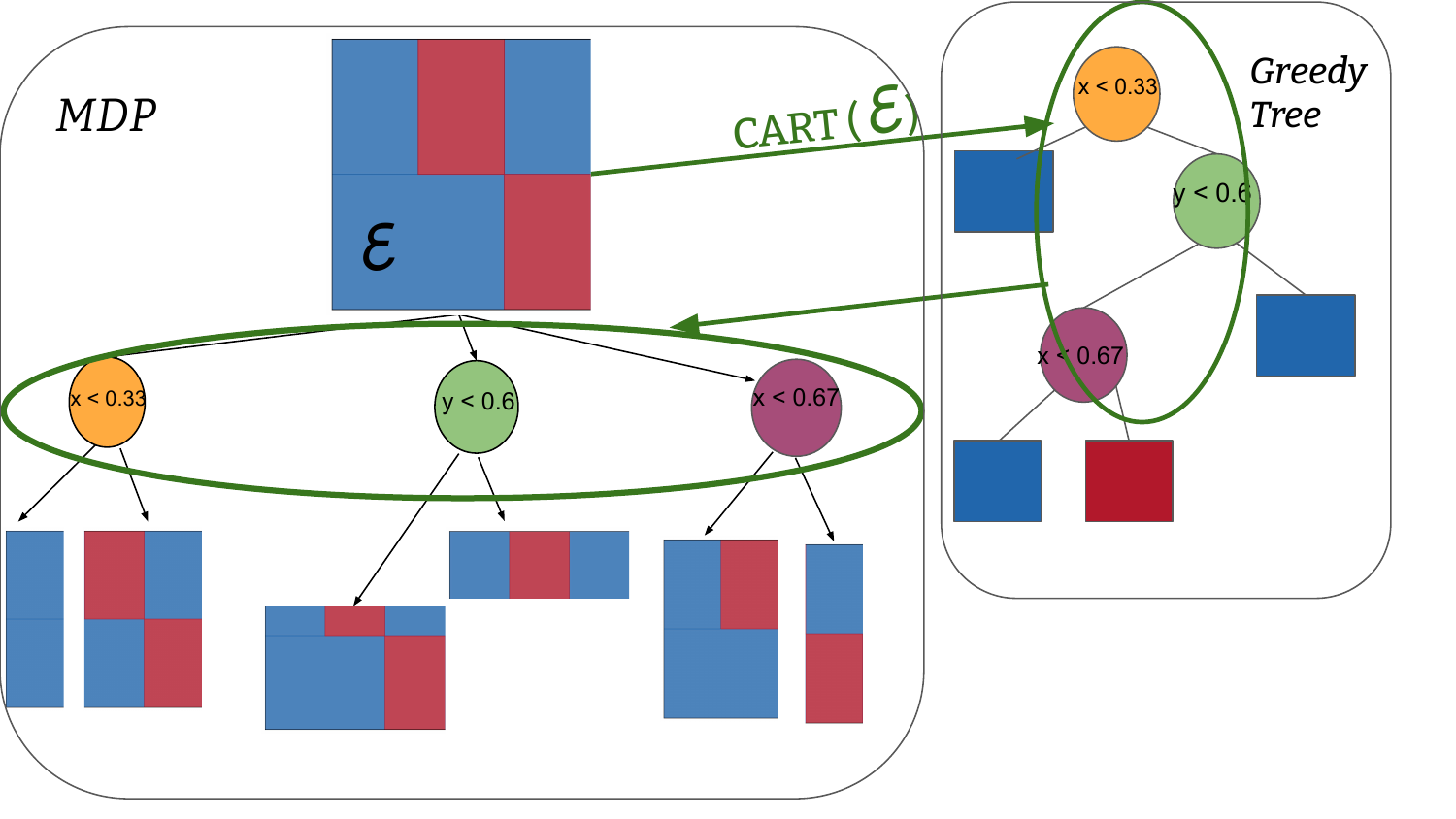}
      \Description{This is a schematic of split generation with calls to CART.}
      \caption{How CART is used in DPDT to generate candidate splits given the example data in the current state.}\label{fig:schema-dpdt}
\end{figure}

\subsection{Dynamic Programming to solve the MDP}
\RestyleAlgo{ruled}
\SetKwComment{Comment}{}{}
        \begin{algorithm}
            \KwData{$\text{Dataset }\mathcal{E}, \text{max depth }D, \text{split function }\phi(), $\\
            $\text{split function parameter } B,  \text{regularizing term }\alpha$}
            \KwResult{$\text{Tree } T$}
            $\mathcal{M} \gets build\_mdp(\mathcal{E}, D, \phi(), B)$\label{line:build_mdp} \\
            \Comment{{// Backward induction}}
            $Q^*(s,a) \gets R_{\alpha}^{\mathcal{M}}(s,a) + \sum_{s'} P^{\mathcal{M}}(s,a,s') \max_{a' \in A_{s'}^{\mathcal{M}}} Q^*(s',a') \forall s,a \in \mathcal{M}$\\
            \Comment{{// Get the optimal policy}}
            $\pi^*(s) = \argmax_{a \in A_s^{\mathcal{M}}} Q^*(s, a) \forall s \in \mathcal{M} $\\
            \Comment{{// Extracting tree from policy}}
            $T \gets E(\pi^*,s_0^{\mathcal{M}}) $
            \caption{DPDT}\label{alg:dpdt}
        \end{algorithm}
        
After constructing the MDP with a chosen splits generating function, we solve for the optimal policy using dynamic programming. Starting from terminal states and working backward to the initial state, we compute the optimal state-action values using Bellman's optimality equation~\cite{BELLMAN1958228}, and then deducing the optimal policy.

From now on, we write DPDT to denote Algorithm \ref{alg:dpdt} when the split function is a call to CART. We discuss key bottlenecks when implementing DPDT in subsequent sections. We now state theoretical results when using DPDT with the CART heuristic. 

\subsection{Performance Guarantees of DPDT}
We now show that: 1) DPDT minimizes the loss from Eq.~\ref{eq:suplearning} at least as well as greedy trees and 2) there exists problems for which DPDT has strictly lower loss than greedy trees. 
As we restrict the action space at a given state $s$ to a subset of all possible split nodes, DPDT is not guaranteed to find the tree minimizing Eq. \ref{eq:suplearning}. However, we are still guaranteed to find trees that are better or equivalent to those induced by CART:
\begin{theorem}[MDP solutions are not worse than the greedy tree]\label{prop:cart}
Let $\pi^*$ be an optimal deterministic policy of the MDP, where the action space at every state is restricted to the top $B$ most informative or discriminative splits. 
Let $T_0$ be the tree induced by CART and $\{T_1,\dots,T_M\}$ all the sub-trees of $T_0$, \footnote{These sub-trees are interesting to consider since they can be returned by common postprocessing operations following a call to CART, that prune some of the nodes from $T_0$. Please see \cite{pruning1} for a review of pruning methods for decision trees.} then for any $\alpha > 0$, 
\[
{\mathcal L}_\alpha(E(\pi^*, s_0)) \leq \min_{0\leq i\leq M}{\mathcal L}_\alpha(T_i)
\]
\end{theorem}

\begin{proof}
Let us first define $C(T)$, the expected number of splits performed by tree $T$ on dataset $\mathcal E$. 
Here $T$ is deduced from policy $\pi$, i.e. $T=E(\pi, s_0)$. $C(T)$ can be defined recursively as $C(T) = 0$ if $T$ is a leaf node, and $C(T) = 1 + p_l C(T_l) + p_r  C(T_r)$, where $T_l = E(\pi, s_l)$ and $T_r = E(\pi, s_r)$. 
In words, when the root of $T$ is a split node, the expected number of splits is one plus the expected number of splits of the left and right sub-trees of the root node.
\end{proof}

It is known that the greedy tree of depth 2 fails to perfectly classify the XOR problem as shown in Figure \ref{fig:patho} and in \citep{Murthy,how-eff}. We aim to show that DPDT is a cheap way to alleviate the weaknesses of greedy trees in this type of problems. The following theorem states that there exist classification problems such that DPDT optimizes the regularized training loss strictly better than greedy algorithms such as CART, ID3 or C4.5.
\begin{theorem}[DPDT can be strictly better than greedy]\label{thm:better_greedy}
There exists a dataset and a depth $D$ such that the DPDT tree $T^{DPDT}_D$ is strictly better than the greedy tree $T^{greedy}_{D}$ , i.e, $\mathcal{L}_{\alpha=0}(T^{greedy}_{D}) > \mathcal{L}_{\alpha=0}(T^{DPDT}_{D})$.
\end{theorem}

The proof of this theorem is given in the next section.

\subsection{Proof of Improvement over CART}\label{proof-improve-opt}
In this section we construct a dataset for which the greedy tree of depth 2 fails to accurately classify data while DPDT with calls to CART as a splits generating function guarantees a strictly better accuracy. The dataset is the XOR pattern like in Figure \ref{fig:patho}. We will first show that greedy tree induction like CART chooses the first split at random and the second split in between the two columns or rows. Then we will quantify the misclassification of the depth-2 greedy tree on the XOR gate. Finally we will show that using the second greedy split as the root of a tree and then building the remaining nodes greedily, i.e. running DPDT with the CART heuristic, strictly decreases the misclassification. 
\begin{definition}[XOR dataset]\label{def:checkerboard}
     Let us defined the XOR dataset as $\mathcal{E}_{XOR} = \{(X_i, Y_i)\}_{i=1}^N$. $X_i = (x_i, y_i) \sim \mathcal{U}([0,1]^2)$ are i.i.d 2-features samples. $Y_i = f(X_i)$ are alternating classes with $f(x,y) = (\lfloor 2x \rfloor + \lfloor 2y \rfloor) \bmod 2$.
\end{definition}

\begin{lemma} The first greedy split is chosen at random on the XOR dataset from definition \ref{def:checkerboard}.
\end{lemma}\label{lem:first-split}
\begin{proof}
Let us consider an arbitrary split $x = x_v$ parallel to the y-axis. The results apply to splits parallel to the x-axis because the XOR pattern is the same when rotated 90 degrees. The split $x_v$ partitions the dataset into two regions $R_{left}$ and $R_{right}$. Since the dataset has two columns and two rows, any rectangular area that spans the whole height $[0,1)$ has the same proportion of class 0 samples and class 1 samples from definition \ref{def:checkerboard}. So in both $R_{left}$ and $R_{right}$ the probabilities of observing class 0 or class 1 at random are $\frac{1}{2}$. Since the class distributions in left and right regions are independent of the split location, all splits have the same objective value when the objective is a measure of information gain like the entropy or the Gini impurity. Hence, the first split in a greedily induced tree is chosen at random.
\end{proof}

\begin{lemma}\label{lem:second-split}
    When the first split is greedy on the XOR dataset from definition \ref{def:checkerboard}, the second greedy splits are chosen perpendicular to the first split at $y=\frac{1}{2}$
\end{lemma}

\begin{proof}
Assume without loss of generality due to symmetries, that the first greedy split is vertical, at \(x=x_v\), with $x_v <= \frac{1}{2}$. This split partitions the unit square into
$R_{left} = [0,x_v)\times[0,1)$ and $R_{right} = [x_v,1)\times[0,1)$. The split $y=\frac{1}{2}$ further partitions $R_{left}$ into $R_{left-down}$ and $R_{left-up}$ with same areas $x_v \times y = \frac{x_v}{2}$. Due to the XOR pattern, there are only samples of class 0 in $R_{left-down}$ and only samples of class 1 in $R_{left-up}$. Hence the the split $y = \frac{1}{2}$ maximizes the information gain in $R_{left}$, hence the second greedy split given an arbitrary first split $x=x_v$ is necessarily $y=\frac{1}{2}$.
\end{proof}
\begin{definition}[Forced- Tree]\label{def:grid-tree}
Let us define the forced-tree as a greedy tree that is forced to make its first split at $y=\frac{1}{2}$.
\end{definition}

\begin{lemma}\label{lem:dpdt-better}
The forced-tree of depth 2 has a 0 loss on the XOR dataset from definition \ref{def:checkerboard} while, with probability $1-\frac{1}{|\mathcal{E}_{XOR}|}$, the greedy tree of depth 2 has strictly positive loss. 
\end{lemma}

\begin{proof}
This is trivial from the definition of the forced tree since if we start with the split $y=\frac{1}{2}$, then clearly CART will correctly split the remaining data. If instead the first split is some  $x_v \neq \frac{1}{2}$ then CART is bound to make an error with only one extra split allowed. Since the first split is chosen at random, from Lemma \ref{lem:first-split}, there are only two splits ($x=\frac{1}{2}$ and $y=\frac{1}{2}$) out of $2 |\mathcal{E}_{XOR}|$ that do not lead to sub-optimality.   
\end{proof}
We can now formally prove theorem \ref{thm:better_greedy}.
\begin{proof}
    By definition of DPDT, all instances of DPDT with the CART nodes parameter $B\geq2$ include the forced-tree from definition \ref{def:grid-tree} in their solution set when applied to the XOR dataset (definition \ref{def:checkerboard}). 
    We know from lemma \ref{lem:dpdt-better} that with high probability, the forced-tree of depth 2 is strictly more accurate than the greedy tree of depth 2 on the XOR dataset. Because we know by proposition \ref{prop:equiv} that DPDT returns the tree with maximal accuracy from its solution set, we can say that DPDT depth-2 trees are strictly better than depth-2 greedy trees returned by e.g. CART on the XOR dataset. 
\end{proof}

\subsection{Practical Implementation}

The key bottlenecks lie in the MDP construction step of DPDT (Section \ref{sec:MDP}). In nature, all decision tree induction algorithms have time complexity exponential in the number of training subsets per tree depth $D$: $O((2B)^D)$, e.g., CART has $O(2^D)$ time complextiy. We already saw that DPDT saves time by not considering all possible tree splits but only $B$ of them. Using state-dependent split generation also allows to generate more or less candidates at different depths of the tree. Indeed, the MDP state $s = (X,d)$ contains the current depth during the MDP construction process. This means that one can control DPDT's time complexity by giving multiple values of maximum nodes: given $(B_1, B_2, ..., B_D)$, the splits generating function in Algorithm~\ref{alg:dpdt} becomes $\phi(s_i) = \phi(X_i, d=1) = nodes(\text{CART}(s, max\_nodes=B_1))$ and  $\phi(s_j) = \phi(X_j, d=2) = nodes(\text{CART}(s, max\_nodes=B_2))$.

Similarly, the space complexity of DPDT is exponential in the space required to store training examples $\mathcal E$. Indeed, the MDP states that DPDT builds in Algorithm~\ref{alg:dpdt} are training samples $X\subseteq \mathcal E$. Hence, the total space required to run DPDT is $O({Np}(2B)^{D})$ where $Np$ is the size of $\mathcal{E}$. In practice, one should implement DPDT in a depth first search manner to obtain a space complexity linear in the size of training set: $O(DNp)$. In practice DPDT builds the MDP from Section~\ref{sec:MDP} by starting from the root and recursively splitting the training set while backpropagating the $Q$-values. This is possible because the MDP we solve has a (not necessarily binary) tree structure (see Figure~\ref{fig:schema-mdp}) and because the $Q$-values of a state only depend on future states. 

We implemented DPDT\footnote{\url{https://github.com/KohlerHECTOR/DPDTreeEstimator}} following scikit-learn API~\cite{sklearn_api} with depth-first search and state-depth-dependent splits generating. 

\section{Empirical Evaluation}

In this section, we empirically demonstrate strong properties of DPDT trees. 
The first part of our experiments focuses on the quality of solutions obtained by DPDT for objective Eq.\ref{eq:suplearning} compared to greedy and optimal trees. We know by theorems \ref{prop:cart} and \ref{thm:better_greedy} that DPDT trees should find better solutions than greedy algorithms for certain problems; but what about real problems?
After showing that DPDT can find optimal trees by considering much less solutions and thus performing orders of magnitude less operations, we will study the generalization capabilities of the latter: do DPDT trees label unseen data accurately?

\subsection{DPDT optimizing capabilities}
\begin{table*}[ht]
    \centering
    \small
        \caption{Comparison of train accuracies of depth-3 trees and number of operations on classification tasks. For DPDT and Top-B, ``light'' configurations have split function parameters (8, 1, 1) ``full'' have parameters (8, 8, 8). We also include the mean train accuracy over 5 deep RL runs. \textbf{Bold} values are optimal accuracies and {\color{blue} blue} values are the largest non-optimal accuracies.}
    \label{tab:tree_comparison_combined}
    \begin{tabular}{l|cc||cc|cc|cc|c||cc|cc|cc}
    \toprule
    & & & \multicolumn{7}{c||}{\textbf{Accuracy}} & \multicolumn{6}{c}{\textbf{Operations}}\\
    \midrule
    & & & Opt & Greedy & \multicolumn{2}{c|}{DPDT} & \multicolumn{2}{c|}{Top-B} & \multicolumn{1}{c||}{Deep RL} & Opt & Greedy & \multicolumn{2}{c|}{DPDT} & \multicolumn{2}{c}{Top-B}\\
    \textbf{Dataset} & N & p & Quant-BnB & CART & light & full & light & full & Custard & Quant-BnB & CART & light & full & light & full \\
    \midrule
    room & 8103 & 16 & \textbf{0.992} & 0.968 & \color{blue} 0.991 & \textbf{0.992} & 0.990 & \textbf{0.992} & 0.715 &$10^6$ & 15 & 286 & 16100 & 111 & 16100 \\
    bean & 10888 & 16  & \textbf{0.871} & 0.777 & 0.812 & \color{blue} 0.853 & 0.804 & 0.841 & 0.182 & 5$\cdot 10^6$ & 15 & 295 & 25900 & 112 & 16800 \\
    eeg & 11984 & 14  & \textbf{0.708} & 0.666 & 0.689 & \color{blue} 0.706 & 0.684 & 0.699 & 0.549 & 2$\cdot 10^6$ & 13 & 289 & 26000 & 95 & 11000 \\
    avila & 10430 & 10  & \textbf{0.585} & 0.532 & \color{blue}0.574 & \textbf{0.585} & 0.563 & 0.572 & 0.409 & 3$\cdot 10^7$ & 9 & 268 & 24700 & 60 & 38900 \\
    magic & 15216 & 10 & \textbf{0.831} & 0.801 & 0.822 & \color{blue} 0.828 & 0.807 & 0.816 & 0.581 &6$\cdot 10^6$ & 15 & 298 & 28000 & 70 & 4190 \\
    htru & 14318 & 8  & \textbf{0.981} & 0.979 & 0.979 & \color{blue}0.980 & 0.979 & \color{blue}0.980 & 0.860 & 6$\cdot 10^7$ & 15 & 295 & 25300 & 55 & 2180 \\
    occup. & 8143 & 5 & \textbf{0.994} & 0.989 & 0.991 & \textbf{0.994} & 0.990 & \color{blue}0.992 & 0.647 & 7$\cdot 10^5$ & 13 & 280 & 16300 & 33 & 510 \\
    skin & 196045 & 3 & \textbf{0.969} & \color{blue}0.966 & \color{blue}0.966 & \color{blue}0.966 & \color{blue}0.966 & \color{blue}0.966 & 0.612 & 7$\cdot 10^4$ & 15 & 301 & 23300 & 20 & 126 \\
    fault & 1552 & 27 & \textbf{0.682} & 0.553 & 0.672 & \color{blue}0.674 & 0.672 & 0.673 & 0.303 & 9$\cdot 10^8$ & 13 & 295 & 24200 & 111 & 16800 \\
    segment & 1848 & 18 & \textbf{0.887} & 0.574 & 0.812 & \color{blue}0.879 & 0.786 & 0.825 & 0.137 & 2$\cdot 10^6$ & 7 & 220 & 16300 & 68 & 11400 \\
    page & 4378 & 10 &  \textbf{0.971} & 0.964 & \color{blue}0.970 & \color{blue}0.970 & 0.964 & 0.965 & 0.902 &$10^7$ & 15 & 298 & 22400 & 701 & 4050 \\
    bidding & 5056 & 9  & \textbf{0.993} & 0.981 & \color{blue}0.985 & \textbf{0.993} & 0.985 & \textbf{0.993} & 0.810 & 3$\cdot 10^5$ & 13 & 256 & 9360 & 58 & 2700 \\
    raisin & 720 & 7 & \textbf{0.894} & 0.869 & 0.879 & \color{blue}0.886 & 0.875 & 0.883 & 0.509 & 4$\cdot 10^6$ & 15 & 295 & 20900 & 48 & 1440 \\
    rice & 3048 & 7 & \textbf{0.938} & 0.933 & 0.934 & \color{blue}0.937 & 0.933 & 0.936 & 0.519 & 2$\cdot 10^7$ & 15 & 298 & 25500 & 49 & 1470 \\
    wilt & 4339 & 5 & \textbf{0.996} & 0.993 & 0.994 & \color{blue}0.995 & 0.994 & 0.994 & 0.984 &3$\cdot 10^5$ & 13 & 274 & 11300 & 33 & 465 \\
    bank & 1097 & 4 & \textbf{0.983} & 0.933 & 0.971 & \color{blue}0.980 & 0.951 & 0.974 & 0.496 & 6$\cdot 10^4$ & 13 & 271 & 7990 & 26 & 256 \\
    \bottomrule
    \end{tabular}
\end{table*}
\begin{figure}
    \centering
    \includegraphics[width=1\linewidth]{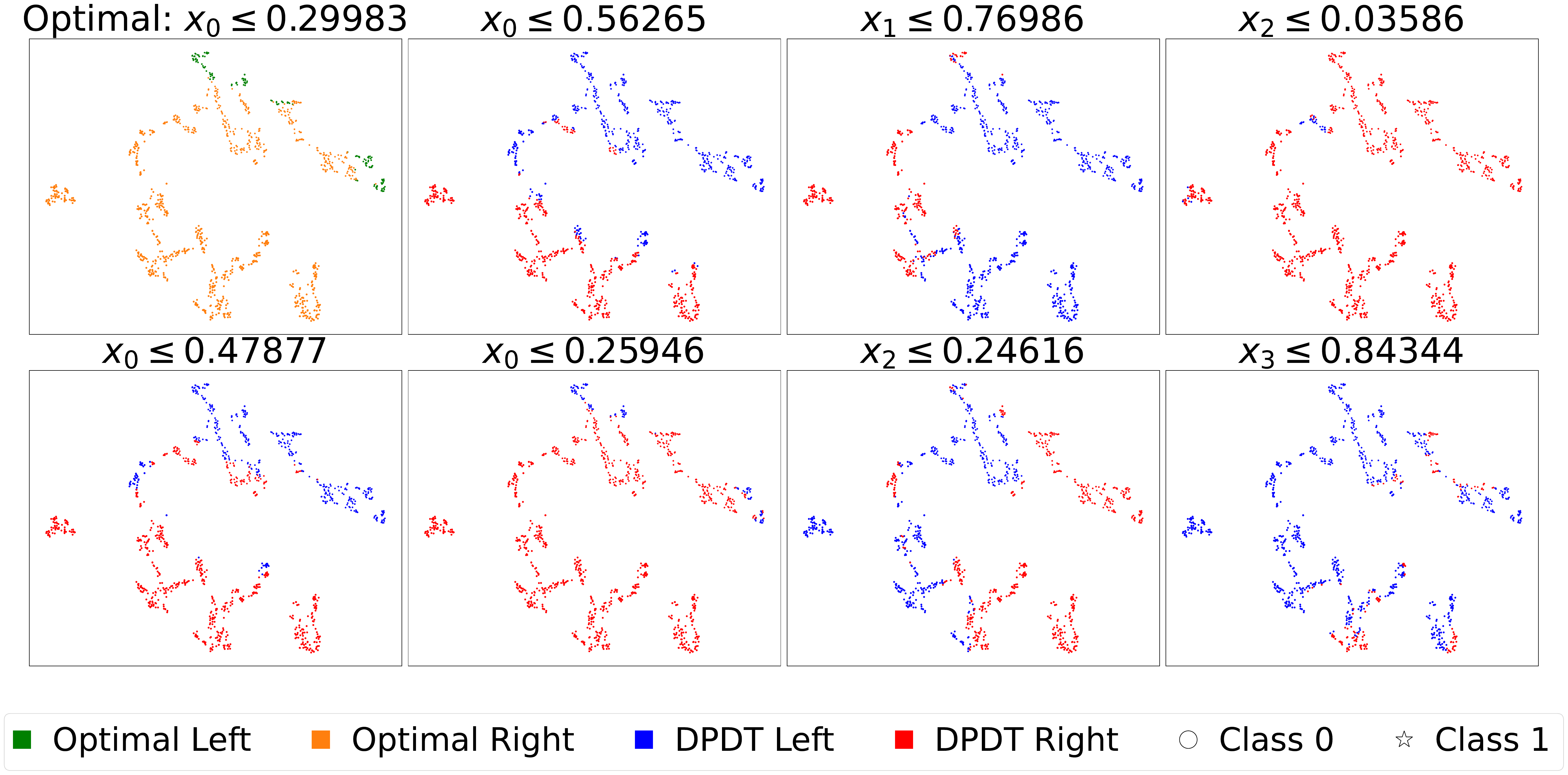}
    \caption{Root splits candidate obtained with DPDT compared to the optimal root split on the Bank dataset. Each split creates a partition of $p$-dimensional data that we projected in the $2$-dimensional space using t-SNE.}
    \Description{This figures shows a point cloud. The point cloud is similar to a clustered one with data from one class in a region and data from an other class in an other region. They are well separated.}
    \label{fig:splits_dpdt}
\end{figure}
\begin{figure}
    \centering
    \includegraphics[width=1\linewidth]{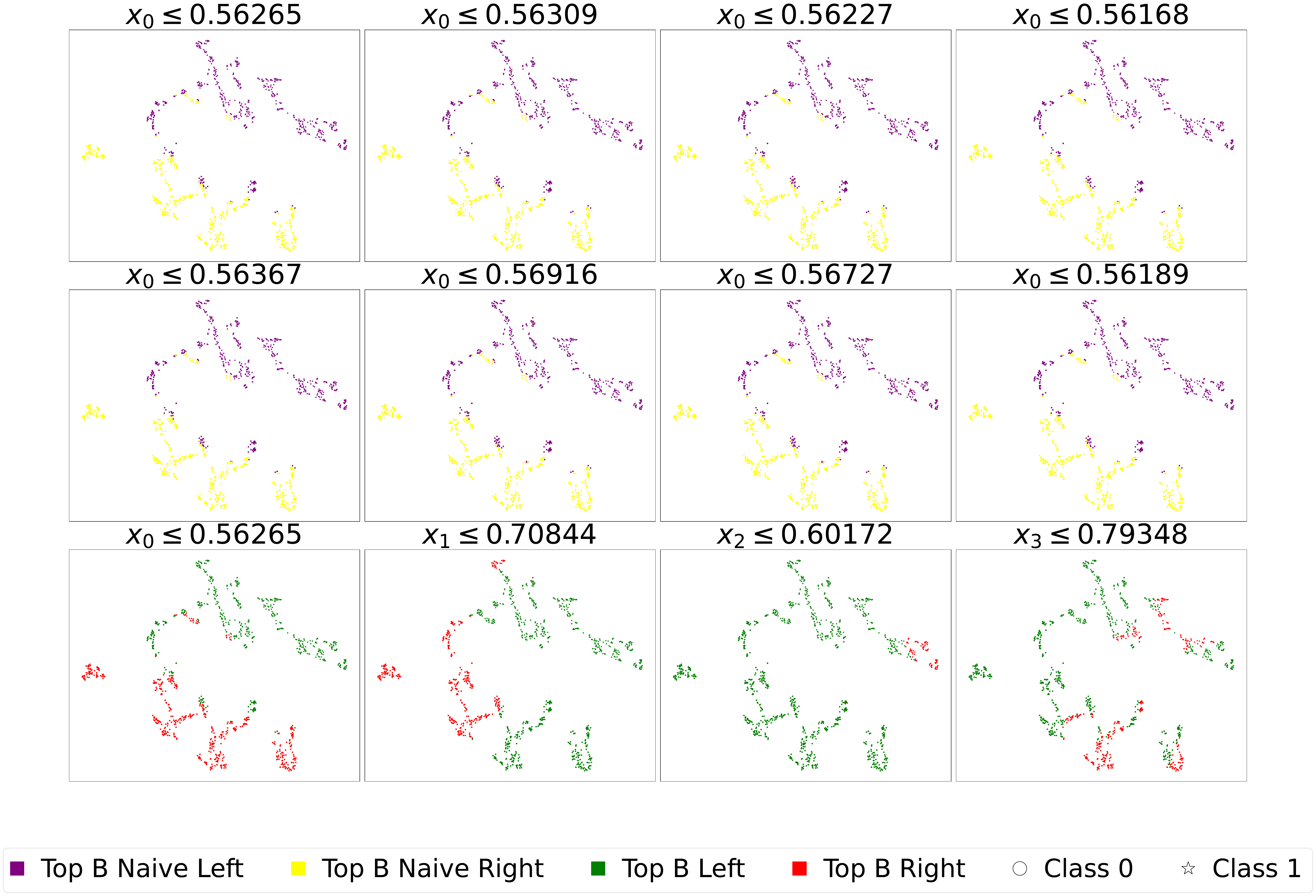}
    \caption{Root splits candidate obtained with Top-B\citep{topk} on the Bank dataset. Each split creates a partition of $p$-dimensional data that we projected using t-SNE.}
    \Description{This figures shows a point cloud. The point cloud is similar to a clustered one with data from one class in a region and data from an other class in an other region. They are well separated.}
    \label{fig:splits_topb}
\end{figure}

From an empirical perspective, it is key to evaluate DPDT training accuracy since optimal decision tree algorithms against which we wish to compare ourselves are designed to optimize the regularized training loss Eq.\ref{eq:suplearning}.

\subsubsection{Setup}
\paragraph{Metrics:} we are interested in the regularized training loss of algorithms optimizing Eq.\ref{eq:suplearning} with $\alpha=0$ and a maximum depth $D$. We are also interested in the number of key operations performed by each baseline, namely computing candidate split nodes for subsets of the training data. We disregard running times as solvers are implemented in different programming languages and/or using optimized code: operations count is more representative of an algorithm efficiency. We also qualitatively compare different decision trees root splits to some optimal root split.
\paragraph{Baselines:} we benchmark DPDT against greedy trees and optimal trees. For greedy trees we compare DPDT to CART \citep{breiman1984classification}. For optimal trees we compare DPDT to Quant-BnB \citep{quantbnb} which is the only solver specialized for depth 3 trees and continuous features. We also consider the non-greedy baseline Top-B \citep{topk}. Ideally, DPDT should have training accuracy close to the optimal tree while performing a number of operations close to the greedy algorithm. Furthermore, comparing DPDT to Top-B brings answers to which heuristic splits are better to consider. 

We use the CART algorithm implemented in \texttt{scikit-learn}~\cite{scikit-learn} in \texttt{CPython} with a maximum depth of 3. Optimal trees are obtained by running the \texttt{Julia} implementation of the Quant-BnB solver from~\cite{quantbnb} specialized in depth 3 trees for datasets with continuous features. We use a time limit of 24 hours per dataset. 
DPDT and Top-B trees are obtained with Algorithm~\ref{alg:dpdt} implemented in pure \texttt{Python} and the calls to CART and Top-B most informative splits generating functions from Section~\ref{sec:MDP} respectively.
We also include Custard, a deep  RL baseline ~\cite{topin2021iterative}. Custard fits a neural network online one datum at a time rather than solving exactly the MDP from Section~\ref{sec:MDP} which states are sets of data. Similarly to DPDT, Custard neural network policy is equivalent to a decision tree. We implement Custard with the DQN agent from \texttt{stable-baselines3}~\cite{stable-baselines3} and train until convergence. 

\paragraph{Datasets:} we us the same datasets as the Quant-BnB paper~\cite{quantbnb}.

\subsubsection{Observations}

\paragraph{Near-optimality.} Our experimental results demonstrate that unlike Deep RL, DPDT and Top-B approaches consistently improve upon greedy solutions while requiring significantly fewer operations than exact solvers. Looking at Table~\ref{tab:tree_comparison_combined}, we observe several key patterns:
first, light DPDT with 16 candidate root splits consistently outperforms the greedy baseline in all datasets. This shows that in practice DPDT can be strictly netter than CART outside of theorem \ref{thm:better_greedy} assumptions. 
Second, when comparing DPDT to Top-B, we see that DPDT generally achieves better accuracy for the same configuration. For example, on the bean dataset, full DPDT reaches 85.3\% accuracy while full Top-B achieves 84.1\%. This pattern holds on most datasets, suggesting that DPDT is more effective than selecting splits based purely on information gain.

Third, both approaches achieve impressive computational efficiency compared to exact solvers. While optimal solutions require between $10^4$ to $10^8$ operations, DPDT and Top-B typically need only $10^2$ to $10^4$ operations, a reduction of 2 to 4 orders of magnitude.
Notably, on several datasets (room, avila, occupancy, bidding), full DPDT matches or comes extremely close to optimal accuracy while requiring far fewer operations. For example, on the room dataset, full DPDT achieves the optimal accuracy of 99.2\% while reducing operations from $1.34\times10^6$ to $1.61\times10^4$.
These results demonstrate that DPDT provides an effective middle ground between greedy approaches and exact solvers, offering near-optimal solutions with reasonable computational requirements. While both DPDT and Top-B improve upon greedy solutions, DPDT CART-based split generation strategy appears to be particularly effective at finding high-quality solutions.

\paragraph{DPDT splits} To understand why the CART-based split generation yields more accurate DPDT trees than the Top-B heuristic, we visualize how splits partition the feature space (Figures \ref{fig:splits_dpdt}, \ref{fig:splits_topb}). We run both DPDT with splits from CART and DPDT with the Top-B most informative splits on the bank dataset. We use t-SNE to create a two-dimensional representations of the dataset partitions given by candidates root splits from CART and Top-B. 
The optimal root split for the depth-3 tree for bank--obtained with Quant-BnB--is shown on Figure \ref{fig:splits_dpdt} in the top-left subplot using green and orange colors for the resulting partitions. On the same figure we can see that the DPDT split generated with CART $x_0 \leq 0.259$ is very similar to the optimal root split. However, on Figure \ref{fig:splits_topb} we observe that no Top-B candidate splits resemble the optimal root and that in general Top-B split lack diversity: they always split along the same feature. We tried to enforce diversity by getting the most informative split \textit{per feature} but no candidate split resembles the optimal root.

\subsection{DPDT generalization capabilities}\label{sec:generalization}
\begin{figure*}
    \centering
    \begin{minipage}{0.24\textwidth}
        \includegraphics[width=\textwidth]{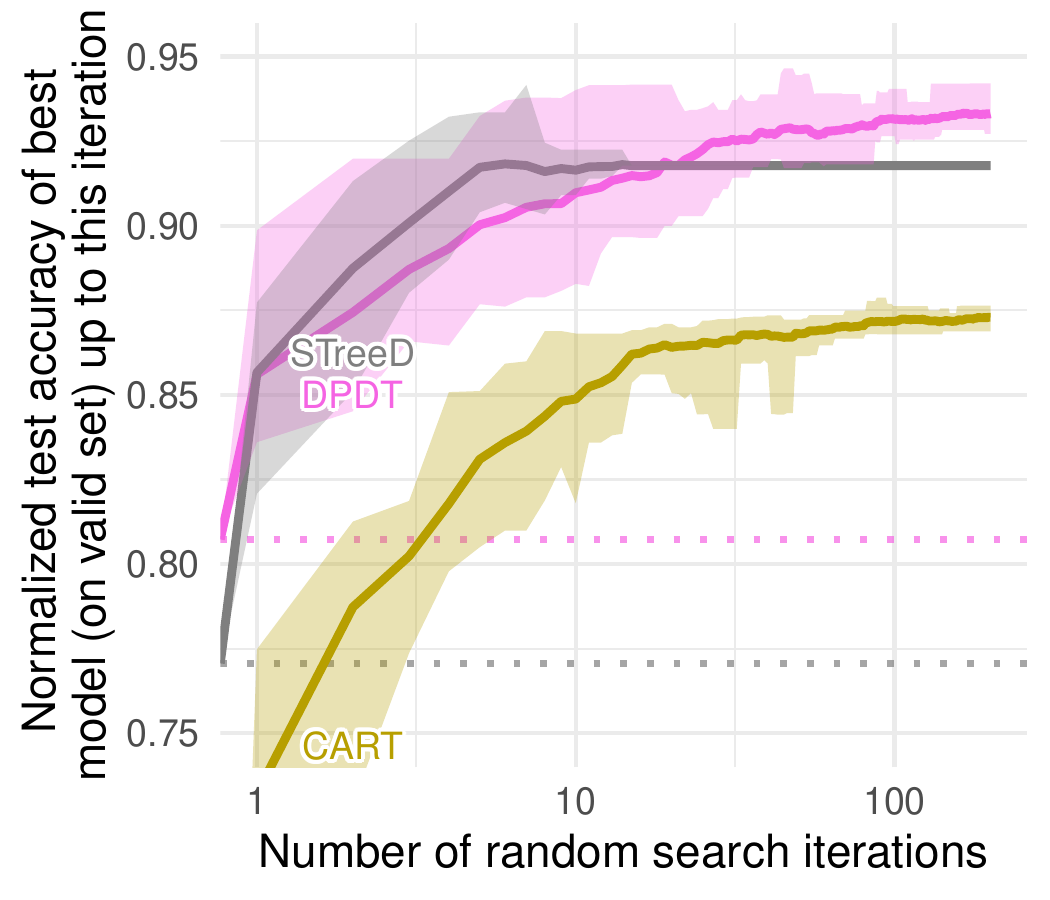}
        \subcaption{Single Tree Numerical}\label{fig:gen-num}
    \end{minipage}
    \begin{minipage}{0.24\textwidth}
        \includegraphics[width=\textwidth]{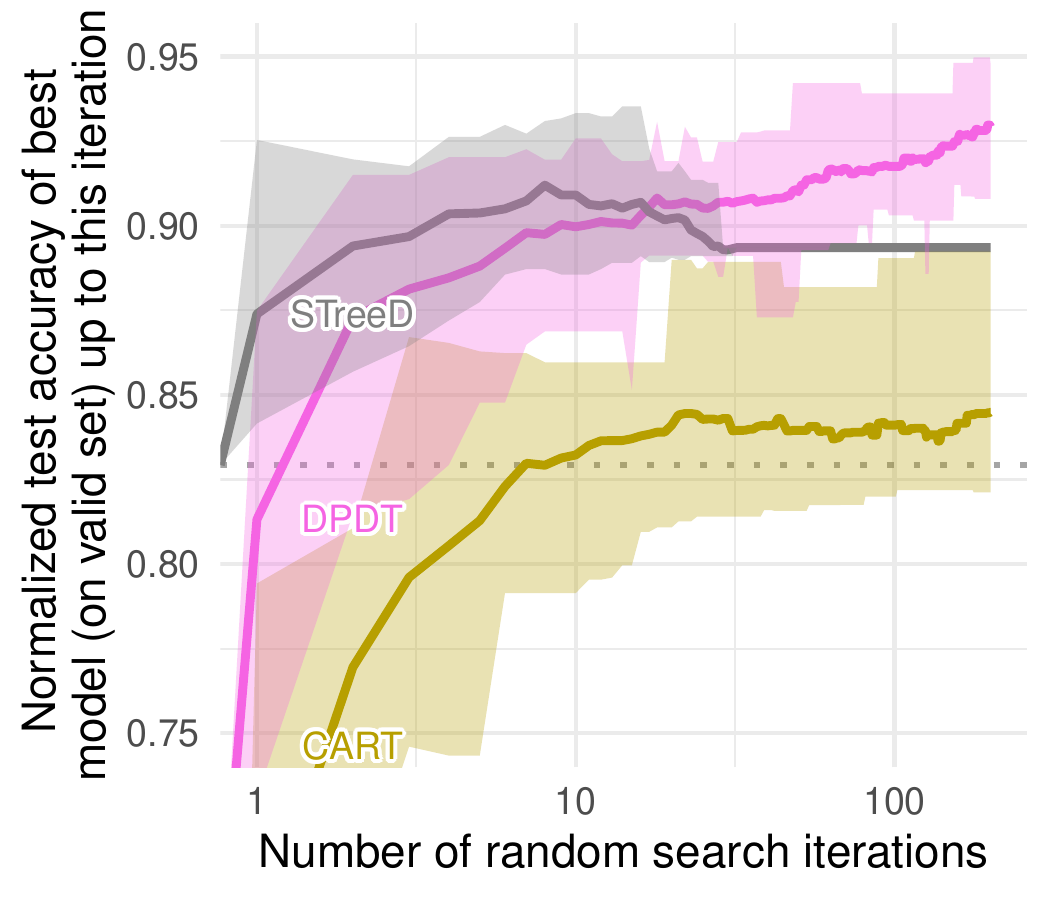}
        \subcaption{Single Tree Categorical}\label{fig:gen-cat}
    \end{minipage}
        \centering
    \begin{minipage}{0.24\textwidth}
        \includegraphics[width=\textwidth]{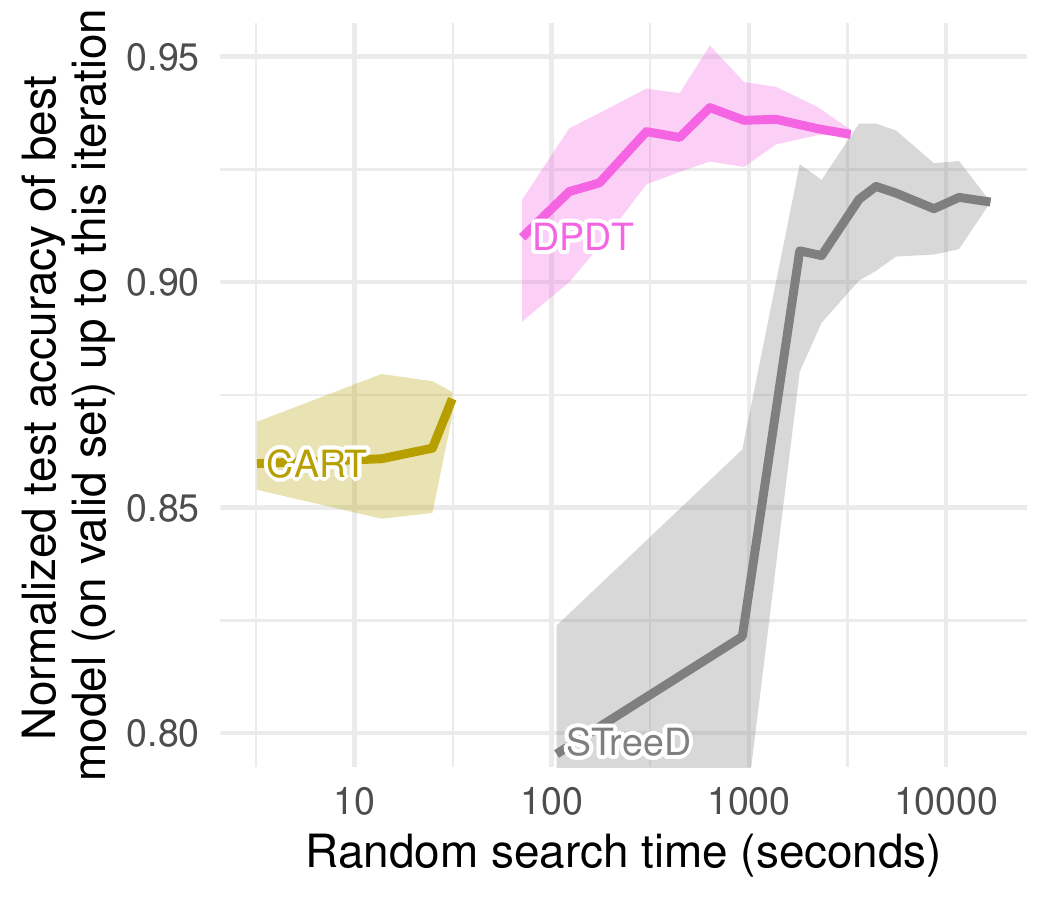}
        \subcaption{Single Tree Numerical}\label{fig:gen-num-time}
    \end{minipage}
    \begin{minipage}{0.24\textwidth}
        \includegraphics[width=\textwidth]{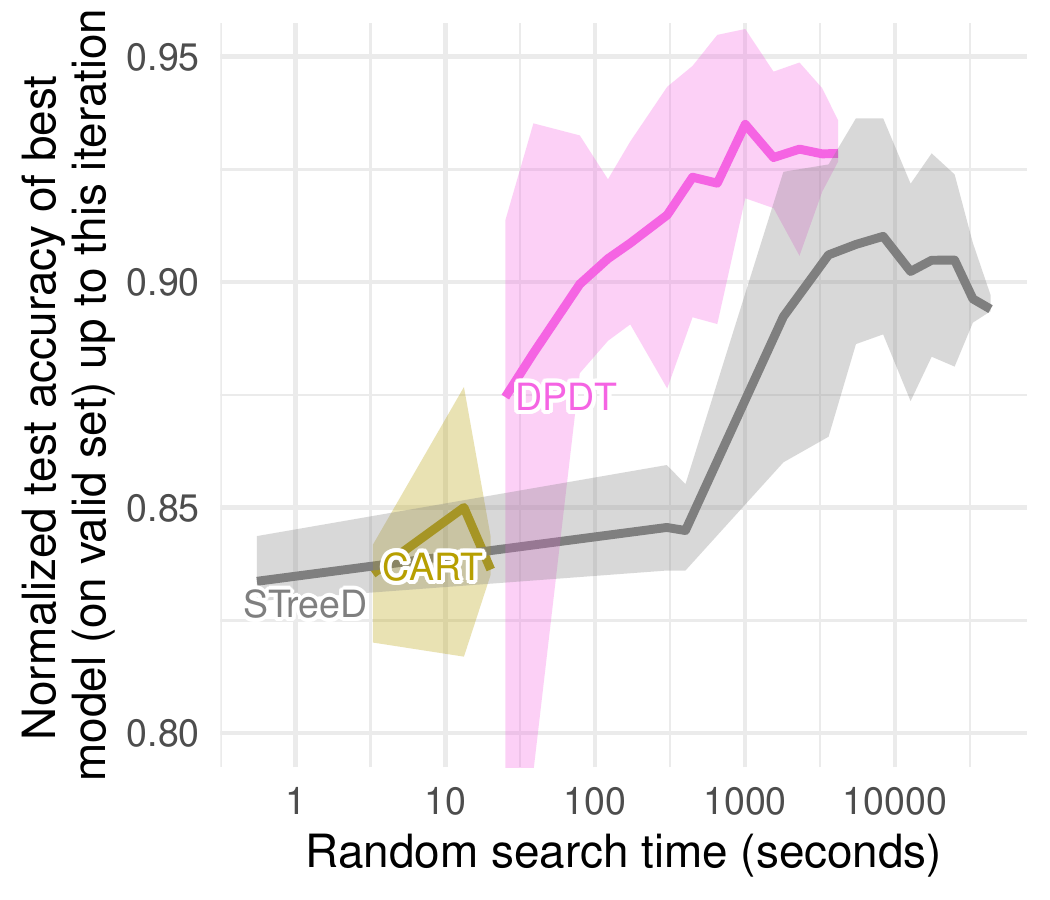}
        \subcaption{Single Tree Categorical}\label{fig:gen-cat-time}
    \end{minipage}
          \begin{minipage}{0.24\textwidth}
          \includegraphics[width=\textwidth]{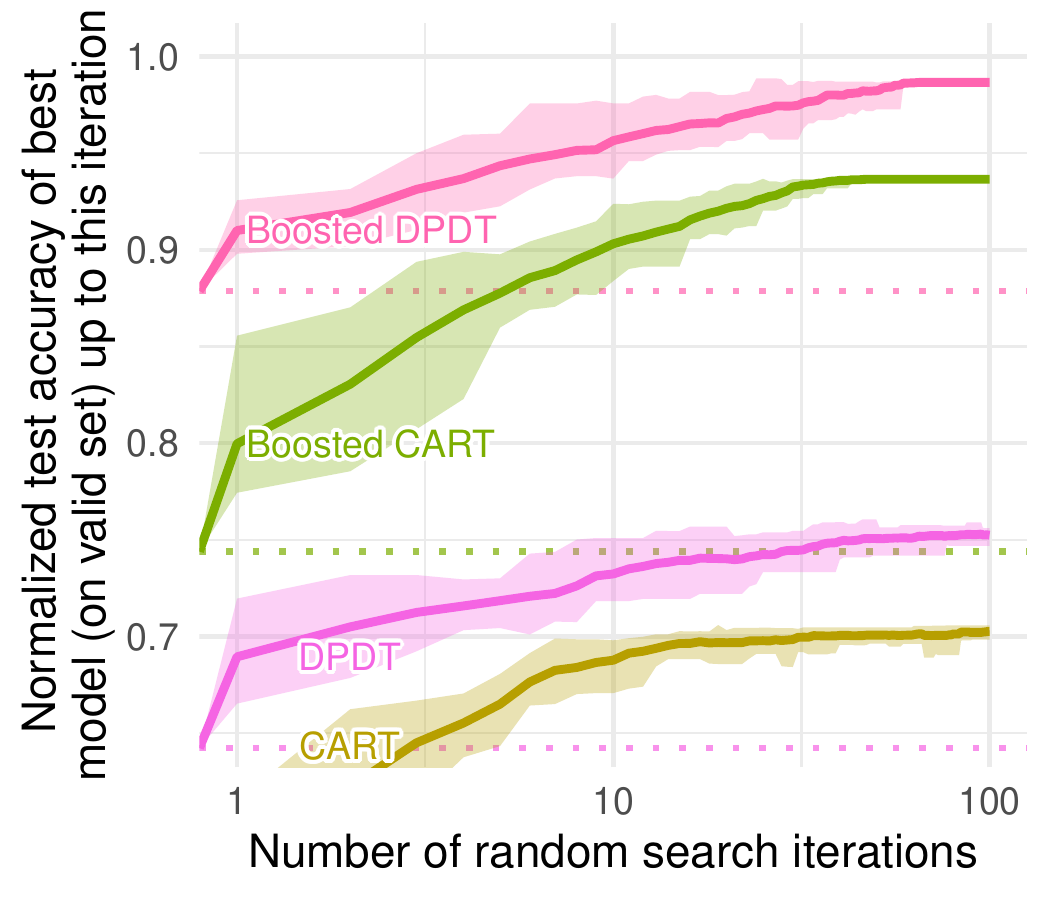}
          \subcaption{Boosting vs Single Tree Num.}\label{fig:boost-num}
      \end{minipage}
      \begin{minipage}{0.24\textwidth}
          \includegraphics[width=\textwidth]{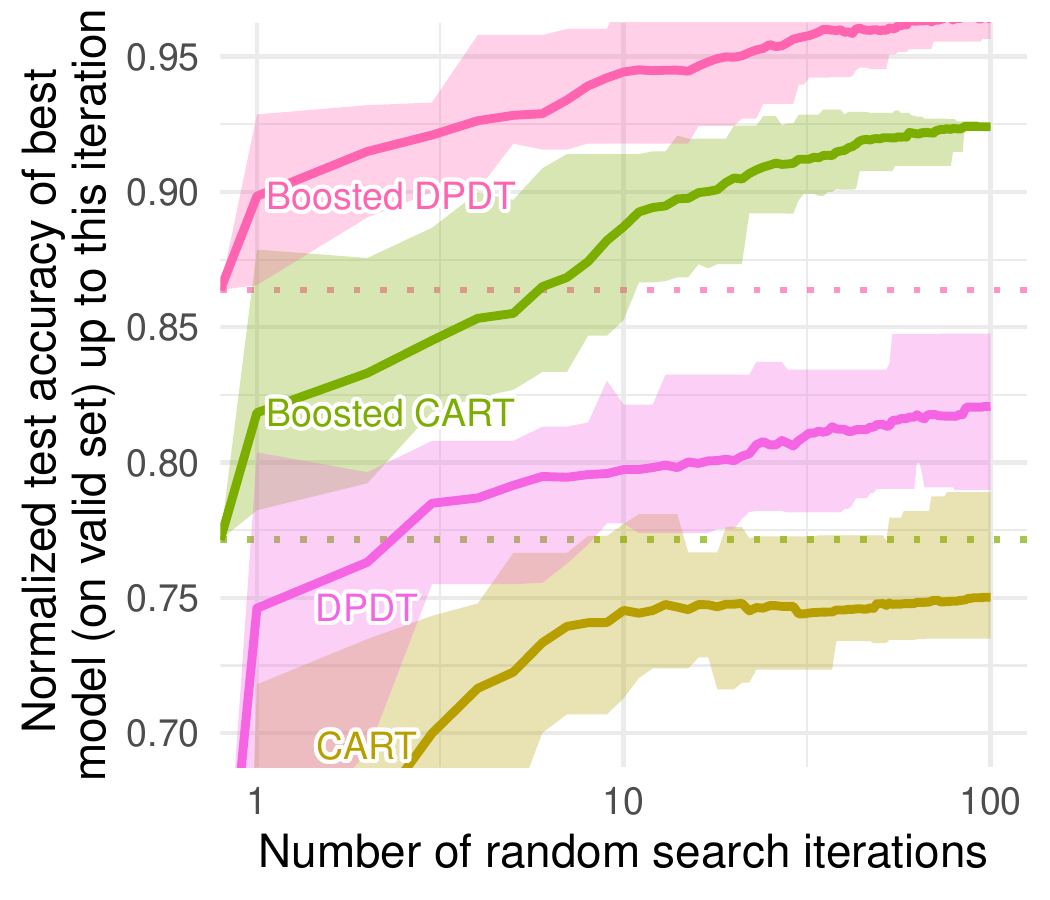}
          \subcaption{Boosting vs Single Tree Cat.}\label{fig:boost-cat}
      \end{minipage}
    \centering
    \begin{minipage}{0.24\textwidth}
        \includegraphics[width=\textwidth]{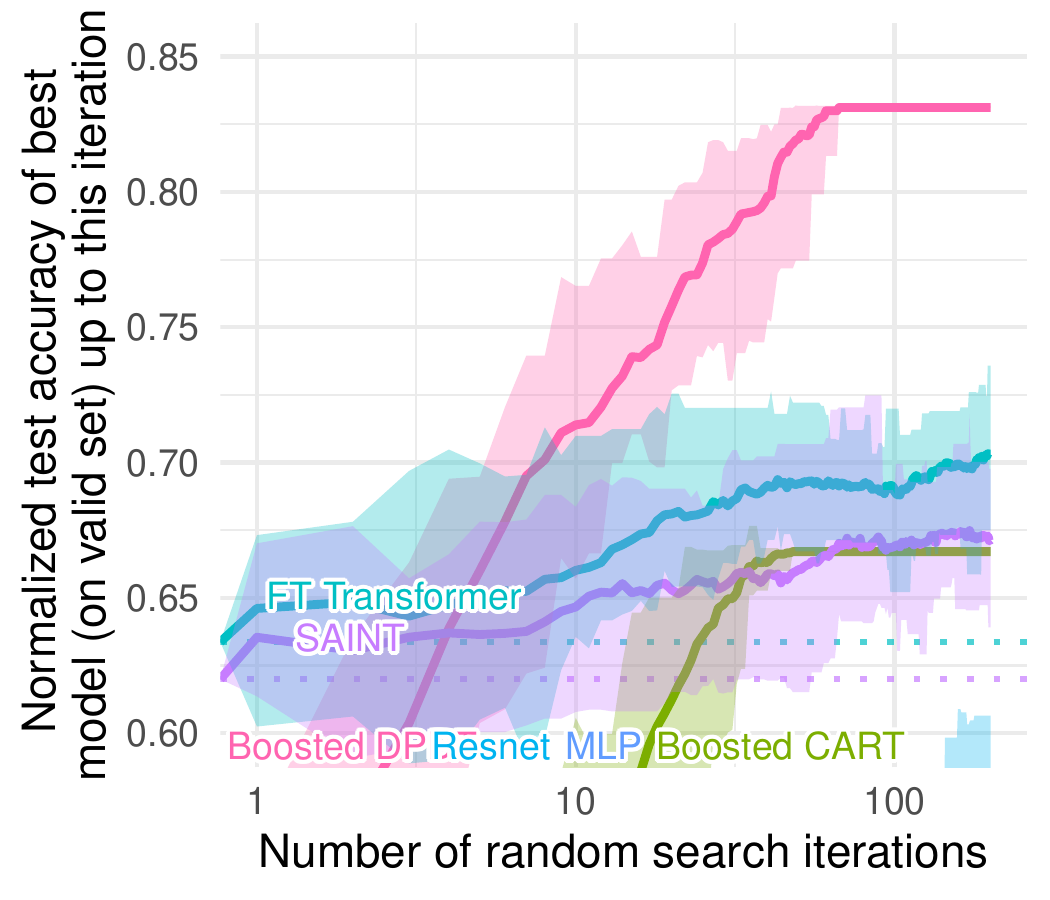}
        \subcaption{Boosting vs Neural Numerical}\label{fig:classif-num-all}
    \end{minipage}
    \begin{minipage}{0.24\textwidth}
        \includegraphics[width=\textwidth]{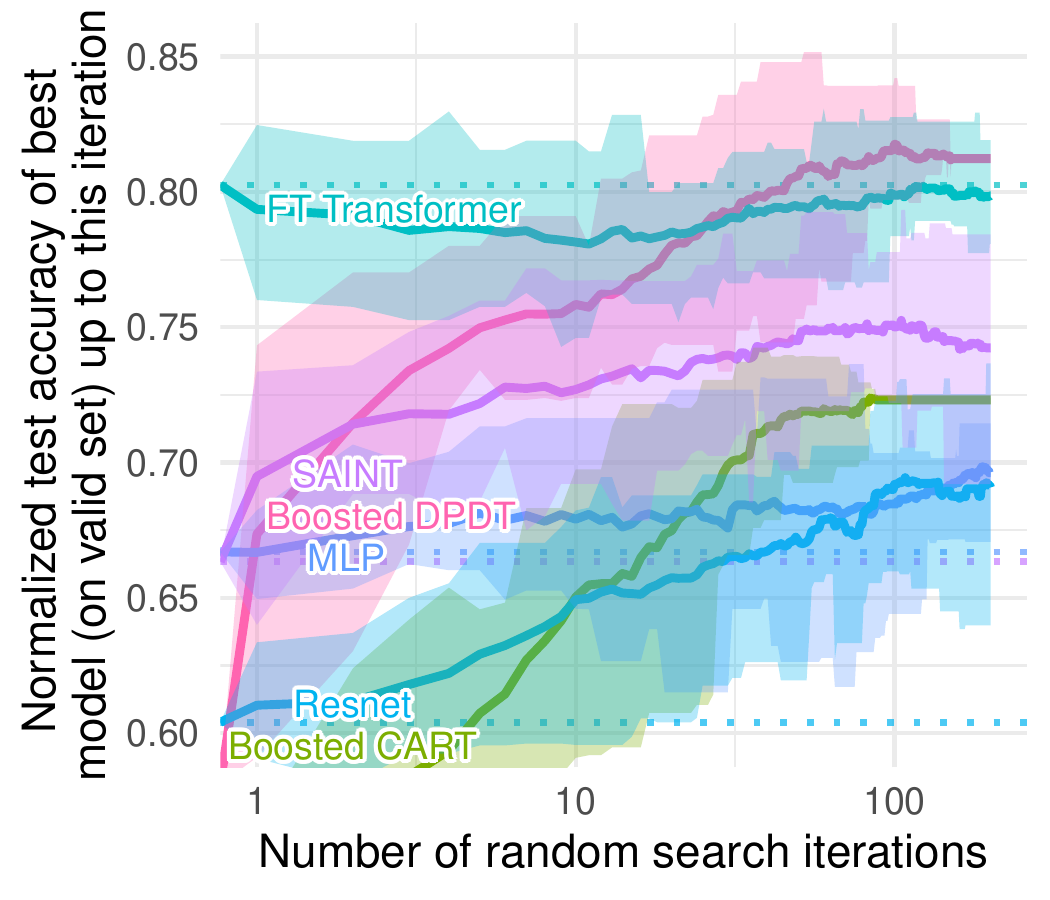}
        \subcaption{Boosting vs Neural Categorical}\label{fig:classif-cat-all}
    \end{minipage}
\caption{Benchmark on medium-sized datasets. Dotted lines correspond to the score of the default hyperparameters, which is also the first random search iteration. Each value corresponds to the test score of the best model (obtained on the validation set) after a specific number of random search iterations (a, b) or after a specific time spent doing random search (c, d), averaged on 15 shuffles of the random search order. The ribbon corresponds to the minimum and maximum scores on these 15 shuffles.}\label{fig:gen-classif}\Description{This figure contains many subplots. Each subplot contains curves that increase then flatten. Those curves show the evolution of the best solution found by decision tree induction algorithms after n iterations of random search.}
\end{figure*}

The goal of this section is to have a fair comparison of generalization capabilities of different tree induction algorithms. Fairness of comparison should take into account the number of hyperparameters, choice of programming language, intrinsic purposes of each algorithms (what are they designed to do?), the type of data they can read (categorical features or numerical features). We benchmark DPDT using~\cite{grinsztajn2022tree}. We choose this benchmark because it was used to establish XGBoost~\citep{xgb} as the SOTA tabular learning model. 

\subsubsection{Setup}

\paragraph{Metrics:} We re-use the code from \citep{grinsztajn2022tree}~\footnote{https://github.com/leogrin/tabular-benchmark}. It relies on random searches for hyper-parameter tuning \citep{pmlr-v28-bergstra13}. We run a random search of 100 iterations per dataset for each benchmarked tree algorithms. To study performance as a function of the number $n$ of random search iterations, we compute the best hyperparameter combination on the validation set on these $n$ iterations (for each model and dataset), and evaluate it on the test set. Following \cite{grinsztajn2022tree}, we do this 15 times while shuffling the random search order at each time. This gives us bootstrap-like estimates of the expected test score of the best tree found on the validation set after each number of random search iterations. In addition, we always start the random searches with the default hyperparameters of each tree induction aglorithm. We use the test set accuracy (classification) to measure model performance. The aggregation metric is discussed in details in \cite[Section 3]{grinsztajn2022tree}.

\paragraph{Datasets:} we use the datasets curated by \cite{grinsztajn2022tree}. They are available on \texttt{OpenML} \citep{10.1145/2641190.2641198} and described in details in \cite[Appendix A.1]{grinsztajn2022tree}. The attributes in these datasets are either 
numerical (a real number), or categorical (a symbolic values among a finite set of possible values). 
The considered datasets follow a strict selection \cite[Section 3]{grinsztajn2022tree} to focus on core learning challenges. Some datasets are very large (millions of samples) like Higgs or Covertype \citep{higgs_280,covertype_31}. To ensure non-trivial learning tasks, datasets where simple models (e.g.\@ logistic regression) performed within 5\% of complex models (e.g.\@ ResNet~\cite{resnet}, HistGradientBoosting~\cite{scikit-learn}) are removed. We use the same data partitioning strategy as~\cite{grinsztajn2022tree}: 70\% of samples are allocated for training, with the remaining 30\% split between validation (30\%) and test (70\%) sets. Both validation and test sets are capped at 50,000 samples for computational efficiency. All algorithms and hyperparameter combinations were evaluated on identical folds. Finally, while we focus on classification datasets in the main text, we provide results for regression problems in table~\ref{tab:regression} in the appendix.

\paragraph{Baselines:}
we benchmark DPDT against CART and STreeD when inducing trees of depth at most 5.  We use hyperparameter search spaces from \cite{komer-proc-scipy-2014} for CART and DPDT. For DPDT we additionally consider eight different splits functions parameters configurations for the maximum nodes in the calls to CART. Surprisingly, after computing the importance of each hyperparameter of DPDT, we found that the maximum node numbers in the calls to CART are only the third most important hyperparametrer behind classical ones like the minimum size of leaf nodes or the minimum impurity decrease (Table~\ref{tab:importance_comparison}). We use the \texttt{CPython} implementation of STreeD\footnote{PySTreeD: \url{https://github.com/AlgTUDelft/pystreed}}. All hyperparameter grids are given in table \ref{tab:tree_hyperparams} in the appendix.

\paragraph{Hardware:} experiments were conducted on a heterogeneous computing infrastructure made of AMD EPYC 7742/7702 64-Core and Intel Xeon processors, with hardware allocation based on availability and algorithm requirements. DPDT and CART random searches ran for the equivalent of 2-days while \texttt{PySTreeD} ran for 10-days.

\begin{table}
\centering
\small
\caption{Hyperparameters importance comparison. A description of the hyperparameters can be found in the scikit-learn documentation: \url{https://scikit-learn.org/stable/modules/generated/sklearn.tree.DecisionTreeClassifier.html}.}
\begin{tabular}{lccc}
\toprule
\textbf{Hyperparameter} & \textbf{DPDT (\%)}  & \textbf{CART (\%)} & \textbf{STreeD (\%)} \\
\midrule
min\_samples\_leaf & 35.05 & 33.50 & 50.50 \\
min\_impurity\_decrease & 24.60 & 24.52 & - \\
cart\_nodes\_list & 15.96 & - & - \\
max\_features & 11.16 & 18.06 & - \\
max\_depth & 7.98 & 10.19 & 0.00 \\
max\_leaf\_nodes & - & 7.84 & - \\
min\_samples\_split & 2.67 & 2.75 & - \\
min\_weight\_fraction\_leaf & 2.58 & 3.14 & - \\
max\_num\_nodes & - & - & 27.51 \\
n\_thresholds & - & - & 21.98 \\
\bottomrule
\end{tabular}
\label{tab:importance_comparison}
\end{table}

\subsubsection{Observations}

\paragraph{Generalization} In Figure \ref{fig:gen-classif}, we observe that DPDT learns better trees than CART and STreeD both in terms of generalization capabilities and in terms of computation cost. On Figures \ref{fig:gen-num} and \ref{fig:gen-cat}, DPDT obtains best generalization scores for classification on numerical and categorical data after 100 iterations of random hyperparameters search over both CART and STreeD. Similarly, we also present generalization scores as a function of compute time (instead of random search iterations). On Figures \ref{fig:gen-num-time} and \ref{fig:gen-cat-time}, despite being coded in the slowest language (\texttt{Python} vs. \texttt{CPython}), our implementation of DPDT finds the best overall model before all STreeD random searches even finish.
The results from Figure \ref{fig:gen-classif} are appealing for machine learning practitioners and data scientists that have to do hyperparameters search to find good models for their data while having computation constrains.

Now that we have shown that DPDT is extremely efficient to learn shallow decision trees that generalize well to unseen data, it is fair to ask if DPDT can also learn deep trees on very large datasets.

\begin{table}
\centering
\caption{Depth-10 decision trees for the KDD 1999 cup dataset.}
\label{tab:model-comparison}
\small
\begin{tabular}{lccc}
\hline
\textbf{Model} & \textbf{Test Accuracy (\%)} & \textbf{Time (s)} & \textbf{Memory (MB)} \\
\hline
DPDT-(4,) & \textbf{91.30} & 339.85 & 5054 \\
DPDT-(4,4,) & \textbf{91.30} & 881.07 & 5054 \\
CART & 91.29 & 25.36 & 1835 \\
GOSDT-$\alpha=0.0005$& 65.47 & 5665.47 & 1167 \\
GOSDT-$\alpha=0.001$ & 65.45 & 5642.85 & 1167 \\
\hline
\end{tabular}
\end{table}

\paragraph{Deeper trees on bigger datasets.} We also stress test DPDT by inducing deep trees of depth 10 for the KDD 1999 cup dataset\footnote{\url{http://kdd.ics.uci.edu/databases/kddcup99/kddcup99.html}}. The training set has 5 million rows and a mix of 80 continuous and categorical features representing network intrusions. We fit DPDT with 4 split candidates for the root node (DPDT-(4,)) and with 4 split candidates for the root and for each of the internal nodes at depth 1 (DPDT-(4,4,)). We compare DPDT to CART with a maximum depth of 10 and to GOSDT\footnote{Code available at: \url{https://github.com/ubc-systopia/gosdt-guesses}}~\cite{McTavish_Zhong_Achermann_Karimalis_Chen_Rudin_Seltzer_2022} with different regularization values $\alpha$. GOSDT first trains a tree ensemble to binarize a dataset and then solve for the optimal decision tree of depth 10 on the binarized problem. In Table~\ref{tab:model-comparison} we report the test accuracy of each tree on the KDD 1999 cup test set. We also report the memory peak during training and the training duration (all experiments are run on the same CPU). We observe that DPDT can improve over CART even for deep trees and large datasets while using reasonable time and memory. Furthermore, Table \ref{tab:model-comparison} highlights the limitation of optimal trees for practical problems when the dataset is not binary. We observed that GOSDT could not find a good binarization of the dataset even when increasing the budget of the tree ensemble up to the point where most of the computations are spent on fitting the ensemble (see more details about this phenomenon in \cite[Section 5.3]{McTavish_Zhong_Achermann_Karimalis_Chen_Rudin_Seltzer_2022}). In table \ref{tab:more-res} in the appendix, we also show that DPDT performs better than optimal trees for natively binary datasets. In the next section we study the performance of boosted DPDT trees.

\section{Application of DPDT to Boosting}

In the race for machine learning algorithms for tabular data, boosting procedures are often considered the go-to methods for classification and regression problems. Boosting algorithms \cite{FREUND1997119,stcohFriedman,FriedmanBoosting} sequentially add weak learners to an ensemble called strong learner. The development of those boosting algorithms has focused on what data to train newly added weak learners \citep{stcohFriedman,FriedmanBoosting},  or on efficient implementation of those algorithms \citep{xgb,10.5555/3327757.3327770}. We show next that Boosted-DPDT (boosting DPDT trees with AdaBoost \cite{FREUND1997119}) improves over recent deep learning algorithms. 

\subsection{Boosted-DPDT}\label{sec:boosting}

We benchmark Boosted-DPDT with the same datasets, metrics, and hardware as in the previous section on single-tree training. Second, we verify the competitiveness of Boosted-DPDT with other models such as deep learning ones (SAINT~\cite{somepalli2021saintimprovedneuralnetworks} and other deep learning architectures from \cite{resnet}). 

On Figures \ref{fig:boost-num} and \ref{fig:boost-cat} we can notice 2 properties of DPDT. First, as in any boosting procedure, Boosted-DPDT outperforms its weak counterpart DPDT. This serves as a sanity check for boosting DPDT trees. Second, it is clear that boosting DPDT trees yields better models than boosting CART trees on both numerical and categorical data. Figures~\ref{fig:classif-num-all} and \ref{fig:classif-cat-all} show that boosting DPDT trees using the default AdaBoost procedure~\cite{FREUND1997119} is enough to learn models outperforming deep learning algorithms on datasets with numerical features and models in the top-tier on datasets with categorical features. This show great promise for models obtained when boosting DPDT trees with more advanced procedures.

\subsection{(X)GB-DPDT}
\begin{figure}
    \centering
    \includegraphics[trim={0 0 0 0},clip,width=0.8\linewidth]{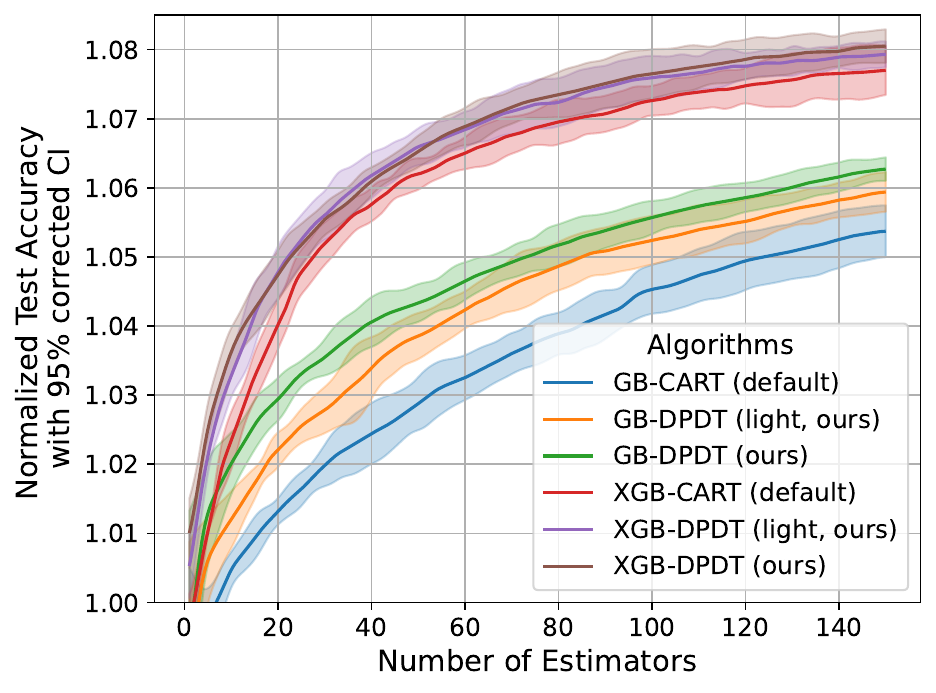}
    \caption{Aggregated mean test accuracies of Gradient Boosting models as a function of the number of single trees.}
    \Description{This figures contains increasing then flattening curves representing the learning process of gradient boosting algorithms.}
    \label{fig:gb}
\end{figure}

We also boost DPDT trees with Gradient Boosting and eXtreme Gradient Boosting~\cite{FriedmanBoosting,stcohFriedman,xgb}(X(GB)-DPDT). For each dataset from~\cite{grinsztajn2022tree}, we trained (X)GB-DPDT models with 150 boosted single DPDT trees and a maximum depth of 3 for each. We evaluate two DPDT configurations for the single trees: light (DPDT-(4, 1, 1)) and the default (DPDT-(4,4,4)). We compare (X)GB-DPDT to (X)GB-CART: 150 boosted CART trees with maximum depth of 3 and default hyperparameters for each. All models use a learning rate of 0.1. For each each dataset, we normalize all boosted models scores by the accuracy of a single depth-3 CART decision tree and aggregate the results: the final curves represent the mean performance across all datasets, with confidence intervals computed using 5 different random seeds.

Figure~\ref{fig:gb} shows that similarly to simple boosting procedures like AdaBoost, more advanced ones like (eXtreme) Gradient Boosting yields better models when the weak learners are DPDT trees rather than greedy trees. This is a motivation to develop efficient implementation of (eXtreme) Gradient Boosting with DPDT as the weak learning algorithm to perform extensive benchmarking following \cite{grinsztajn2022tree} and potentially claim the state-of-the-art.

\section{Conclusion}

In this paper, we introduced Dynamic Programming Decision Trees (DPDT), a novel framework that bridges the gap between greedy and optimal decision tree algorithms. By formulating tree induction as an MDP and employing adaptive split generation based on CART, DPDT achieves near-optimal training loss with significantly reduced computational complexity compared to existing optimal tree solvers. Furthermore, we prove that DPDT can learn strictly more accurate trees than CART. 

Most importantly, extensive benchmarking on varied large and difficult enough datasets showed that DPDT trees and boosted DPDT trees generalize better than other baselines. To conclude, we showed that DPDT is a promising machine learning algorithm. 

The key future work would be to make DPDT industry-ready by implementing it in \texttt{C} and or making it compatible with the most advanced implementation of e.g.\@ XGBoost.

\section{Acknowledgments}

Philippe Preux acknowledges the R\'egion Hauts-de-France CPER project CornelIA. Experiments presented in this paper were carried out using the Grid'5000 testbed (\url{https://www.grid5000.fr}). Hector Kohler acknowledges the AI\_PhD@Lille ANR funding.
\bibliographystyle{ACM-Reference-Format}
\bibliography{sample-base}


\begin{thebibliography}{62}


\ifx \showCODEN    \undefined \def \showCODEN     #1{\unskip}     \fi
\ifx \showISBNx    \undefined \def \showISBNx     #1{\unskip}     \fi
\ifx \showISBNxiii \undefined \def \showISBNxiii  #1{\unskip}     \fi
\ifx \showISSN     \undefined \def \showISSN      #1{\unskip}     \fi
\ifx \showLCCN     \undefined \def \showLCCN      #1{\unskip}     \fi
\ifx \shownote     \undefined \def \shownote      #1{#1}          \fi
\ifx \showarticletitle \undefined \def \showarticletitle #1{#1}   \fi
\ifx \showURL      \undefined \def \showURL       {\relax}        \fi
\providecommand\bibfield[2]{#2}
\providecommand\bibinfo[2]{#2}
\providecommand\natexlab[1]{#1}
\providecommand\showeprint[2][]{arXiv:#2}

\bibitem[Aghaei et~al\mbox{.}(2020)]%
        {mfoct}
\bibfield{author}{\bibinfo{person}{Sina Aghaei}, \bibinfo{person}{Andres
  Gomez}, {and} \bibinfo{person}{Phebe Vayanos}.}
  \bibinfo{year}{2020}\natexlab{}.
\newblock \bibinfo{title}{Learning Optimal Classification Trees: Strong
  Max-Flow Formulations}.
\newblock
\showeprint[arxiv]{2002.09142}~[stat.ML]


\bibitem[Bastani et~al\mbox{.}(2018)]%
        {viper}
\bibfield{author}{\bibinfo{person}{Osbert Bastani}, \bibinfo{person}{Yewen Pu},
  {and} \bibinfo{person}{Armando Solar-Lezama}.}
  \bibinfo{year}{2018}\natexlab{}.
\newblock \showarticletitle{Verifiable Reinforcement Learning via Policy
  Extraction}. In \bibinfo{booktitle}{\emph{NeurIPS}}.
\newblock


\bibitem[Bellman(1958)]%
        {BELLMAN1958228}
\bibfield{author}{\bibinfo{person}{Richard Bellman}.}
  \bibinfo{year}{1958}\natexlab{}.
\newblock \showarticletitle{Dynamic programming and stochastic control
  processes}.
\newblock \bibinfo{journal}{\emph{Information and Control}}
  (\bibinfo{year}{1958}).
\newblock


\bibitem[Bergstra et~al\mbox{.}(2013)]%
        {pmlr-v28-bergstra13}
\bibfield{author}{\bibinfo{person}{James Bergstra}, \bibinfo{person}{Daniel
  Yamins}, {and} \bibinfo{person}{David Cox}.} \bibinfo{year}{2013}\natexlab{}.
\newblock \showarticletitle{Making a Science of Model Search: Hyperparameter
  Optimization in Hundreds of Dimensions for Vision Architectures}. In
  \bibinfo{booktitle}{\emph{ICML}}.
\newblock


\bibitem[Bertsimas and Dunn(2017)]%
        {oct}
\bibfield{author}{\bibinfo{person}{Dimitris Bertsimas} {and}
  \bibinfo{person}{Jack Dunn}.} \bibinfo{year}{2017}\natexlab{}.
\newblock \showarticletitle{Optimal classification trees}.
\newblock \bibinfo{journal}{\emph{Machine Learning}} (\bibinfo{year}{2017}).
\newblock


\bibitem[Blackard(1998)]%
        {covertype_31}
\bibfield{author}{\bibinfo{person}{Jock Blackard}.}
  \bibinfo{year}{1998}\natexlab{}.
\newblock \bibinfo{title}{{Covertype}}.
\newblock \bibinfo{howpublished}{UCI Machine Learning Repository}.
\newblock


\bibitem[Blanc et~al\mbox{.}(2023)]%
        {topk}
\bibfield{author}{\bibinfo{person}{Guy Blanc}, \bibinfo{person}{Jane Lange},
  \bibinfo{person}{Chirag Pabbaraju}, \bibinfo{person}{Colin Sullivan},
  \bibinfo{person}{Li-Yang Tan}, {and} \bibinfo{person}{Mo Tiwari}.}
  \bibinfo{year}{2023}\natexlab{}.
\newblock \showarticletitle{Harnessing the power of choices in decision tree
  learning}. In \bibinfo{booktitle}{\emph{NeurIPS}}.
\newblock


\bibitem[Breiman et~al\mbox{.}(1984)]%
        {breiman1984classification}
\bibfield{author}{\bibinfo{person}{L Breiman}, \bibinfo{person}{JH Friedman},
  \bibinfo{person}{R Olshen}, {and} \bibinfo{person}{CJ Stone}.}
  \bibinfo{year}{1984}\natexlab{}.
\newblock \bibinfo{booktitle}{\emph{Classification and Regression Trees}}.
\newblock \bibinfo{publisher}{Wadsworth}.
\newblock


\bibitem[{B}rent {K}omer et~al\mbox{.}(2014)]%
        {komer-proc-scipy-2014}
\bibfield{author}{\bibinfo{person}{{B}rent {K}omer}, \bibinfo{person}{{J}ames
  {B}ergstra}, {and} \bibinfo{person}{{C}hris {E}liasmith}.}
  \bibinfo{year}{2014}\natexlab{}.
\newblock \showarticletitle{{H}yperopt-{S}klearn: {A}utomatic {H}yperparameter
  {C}onfiguration for {S}cikit-{L}earn}. In
  \bibinfo{booktitle}{\emph{{P}roceedings of the 13th {P}ython in {S}cience
  {C}onference}}.
\newblock


\bibitem[Bressan et~al\mbox{.}(2024)]%
        {pmlr-v247-bressan24a}
\bibfield{author}{\bibinfo{person}{Marco Bressan}, \bibinfo{person}{Nicol{\`o}
  Cesa-Bianchi}, \bibinfo{person}{Emmanuel Esposito}, \bibinfo{person}{Yishay
  Mansour}, \bibinfo{person}{Shay Moran}, {and} \bibinfo{person}{Maximilian
  Thiessen}.} \bibinfo{year}{2024}\natexlab{}.
\newblock \showarticletitle{A Theory of Interpretable Approximations}. In
  \bibinfo{booktitle}{\emph{COLT}}.
\newblock


\bibitem[Buitinck et~al\mbox{.}(2013)]%
        {sklearn_api}
\bibfield{author}{\bibinfo{person}{Lars Buitinck}, \bibinfo{person}{Gilles
  Louppe}, \bibinfo{person}{Mathieu Blondel}, \bibinfo{person}{Fabian
  Pedregosa}, \bibinfo{person}{Andreas Mueller}, \bibinfo{person}{Olivier
  Grisel}, \bibinfo{person}{Vlad Niculae}, \bibinfo{person}{Peter
  Prettenhofer}, \bibinfo{person}{Alexandre Gramfort}, \bibinfo{person}{Jaques
  Grobler}, \bibinfo{person}{Robert Layton}, \bibinfo{person}{Jake VanderPlas},
  \bibinfo{person}{Arnaud Joly}, \bibinfo{person}{Brian Holt}, {and}
  \bibinfo{person}{Ga{\"{e}}l Varoquaux}.} \bibinfo{year}{2013}\natexlab{}.
\newblock \showarticletitle{{API} design for machine learning software:
  experiences from the scikit-learn project}. In \bibinfo{booktitle}{\emph{ECML
  PKDD Workshop: Languages for Data Mining and Machine Learning}}.
\newblock


\bibitem[Carreira-Perpi\~{n}\'{a}n and Zharmagambetov(2020)]%
        {10.1145/3412815.3416882}
\bibfield{author}{\bibinfo{person}{Miguel~\'{A} Carreira-Perpi\~{n}\'{a}n}
  {and} \bibinfo{person}{Arman Zharmagambetov}.}
  \bibinfo{year}{2020}\natexlab{}.
\newblock \showarticletitle{Ensembles of Bagged TAO Trees Consistently Improve
  over Random Forests, AdaBoost and Gradient Boosting}. In
  \bibinfo{booktitle}{\emph{FODS}}.
\newblock


\bibitem[Carreira-Perpinan and Tavallali(2018)]%
        {NEURIPS2018_185c29dc}
\bibfield{author}{\bibinfo{person}{Miguel~A. Carreira-Perpinan} {and}
  \bibinfo{person}{Pooya Tavallali}.} \bibinfo{year}{2018}\natexlab{}.
\newblock \showarticletitle{Alternating optimization of decision trees, with
  application to learning sparse oblique trees}. In
  \bibinfo{booktitle}{\emph{NeurIPS}}.
\newblock


\bibitem[Chaouki et~al\mbox{.}(2025)]%
        {chaouki2024branchesfastdynamicprogramming}
\bibfield{author}{\bibinfo{person}{Ayman Chaouki}, \bibinfo{person}{Jesse
  Read}, {and} \bibinfo{person}{Albert Bifet}.}
  \bibinfo{year}{2025}\natexlab{}.
\newblock \showarticletitle{Branches: A Fast Dynamic Programming and Branch
  {\&} Bound Algorithm for Optimal Decision Trees}. In
  \bibinfo{booktitle}{\emph{ICML}}.
\newblock


\bibitem[Chen and Guestrin(2016)]%
        {xgb}
\bibfield{author}{\bibinfo{person}{Tianqi Chen} {and} \bibinfo{person}{Carlos
  Guestrin}.} \bibinfo{year}{2016}\natexlab{}.
\newblock \showarticletitle{XGBoost: A Scalable Tree Boosting System}. In
  \bibinfo{booktitle}{\emph{KDD}}.
\newblock


\bibitem[Costa and Pedreira(2023)]%
        {costa2023recent}
\bibfield{author}{\bibinfo{person}{Vin{\'\i}cius~G Costa} {and}
  \bibinfo{person}{Carlos~E Pedreira}.} \bibinfo{year}{2023}\natexlab{}.
\newblock \showarticletitle{Recent advances in decision trees: An updated
  survey}.
\newblock \bibinfo{journal}{\emph{Artificial Intelligence Review}}
  (\bibinfo{year}{2023}).
\newblock


\bibitem[Demirovi\'{c} et~al\mbox{.}(2023)]%
        {blossom}
\bibfield{author}{\bibinfo{person}{Emir Demirovi\'{c}},
  \bibinfo{person}{Emmanuel Hebrard}, {and} \bibinfo{person}{Louis Jean}.}
  \bibinfo{year}{2023}\natexlab{}.
\newblock \showarticletitle{Blossom: an Anytime Algorithm for Computing Optimal
  Decision Trees}. In \bibinfo{booktitle}{\emph{ICML}}.
\newblock


\bibitem[Demirovic et~al\mbox{.}(2022)]%
        {murtree}
\bibfield{author}{\bibinfo{person}{Emir Demirovic}, \bibinfo{person}{Anna
  Lukina}, \bibinfo{person}{Emmanuel Hebrard}, \bibinfo{person}{Jeffrey Chan},
  \bibinfo{person}{James Bailey}, \bibinfo{person}{Christopher Leckie},
  \bibinfo{person}{Kotagiri Ramamohanarao}, {and} \bibinfo{person}{Peter~J.
  Stuckey}.} \bibinfo{year}{2022}\natexlab{}.
\newblock \showarticletitle{MurTree: Optimal Decision Trees via Dynamic
  Programming and Search}.
\newblock \bibinfo{journal}{\emph{Journal of Machine Learning Research}}
  (\bibinfo{year}{2022}).
\newblock


\bibitem[Doshi-Velez and Kim(2017)]%
        {rigourous}
\bibfield{author}{\bibinfo{person}{Finale Doshi-Velez} {and}
  \bibinfo{person}{Been Kim}.} \bibinfo{year}{2017}\natexlab{}.
\newblock \bibinfo{title}{Towards A Rigorous Science of Interpretable Machine
  Learning}.
\newblock
\showeprint[arxiv]{1702.08608}~[stat.ML]


\bibitem[Dulac-Arnold et~al\mbox{.}(2012)]%
        {Dulac_Arnold_2011}
\bibfield{author}{\bibinfo{person}{Gabriel Dulac-Arnold},
  \bibinfo{person}{Ludovic Denoyer}, \bibinfo{person}{Philippe Preux}, {and}
  \bibinfo{person}{Patrick Gallinari}.} \bibinfo{year}{2012}\natexlab{}.
\newblock \showarticletitle{Sequential approaches for learning datum-wise
  sparse representations}.
\newblock \bibinfo{journal}{\emph{Machine Learning}} (\bibinfo{year}{2012}).
\newblock


\bibitem[Esposito et~al\mbox{.}(1997)]%
        {pruning1}
\bibfield{author}{\bibinfo{person}{Floriana Esposito}, \bibinfo{person}{Donato
  Malerba}, \bibinfo{person}{Giovanni Semeraro}, {and} \bibinfo{person}{J
  Kay}.} \bibinfo{year}{1997}\natexlab{}.
\newblock \showarticletitle{A comparative analysis of methods for pruning
  decision trees}.
\newblock \bibinfo{journal}{\emph{IEEE transactions on pattern analysis and
  machine intelligence}} (\bibinfo{year}{1997}).
\newblock


\bibitem[Freund and Schapire(1997)]%
        {FREUND1997119}
\bibfield{author}{\bibinfo{person}{Yoav Freund} {and} \bibinfo{person}{Robert~E
  Schapire}.} \bibinfo{year}{1997}\natexlab{}.
\newblock \showarticletitle{A Decision-Theoretic Generalization of On-Line
  Learning and an Application to Boosting}.
\newblock \bibinfo{journal}{\emph{J. Comput. System Sci.}}
  (\bibinfo{year}{1997}).
\newblock


\bibitem[Friedman(2001)]%
        {FriedmanBoosting}
\bibfield{author}{\bibinfo{person}{Jerome~H. Friedman}.}
  \bibinfo{year}{2001}\natexlab{}.
\newblock \showarticletitle{Greedy Function Approximation: A Gradient Boosting
  Machine}.
\newblock \bibinfo{journal}{\emph{The Annals of Statistics}}
  (\bibinfo{year}{2001}).
\newblock


\bibitem[Friedman(2002)]%
        {stcohFriedman}
\bibfield{author}{\bibinfo{person}{Jerome~H. Friedman}.}
  \bibinfo{year}{2002}\natexlab{}.
\newblock \showarticletitle{Stochastic gradient boosting}.
\newblock \bibinfo{journal}{\emph{Comput. Stat. Data Anal.}}
  (\bibinfo{year}{2002}).
\newblock


\bibitem[Garlapati et~al\mbox{.}(2015)]%
        {garlapati2015reinforcementlearningapproachonline}
\bibfield{author}{\bibinfo{person}{Abhinav Garlapati}, \bibinfo{person}{Aditi
  Raghunathan}, \bibinfo{person}{Vaishnavh Nagarajan}, {and}
  \bibinfo{person}{Balaraman Ravindran}.} \bibinfo{year}{2015}\natexlab{}.
\newblock \bibinfo{title}{A Reinforcement Learning Approach to Online Learning
  of Decision Trees}.
\newblock
\showeprint[arxiv]{1507.06923}~[cs.LG]


\bibitem[Gorishniy et~al\mbox{.}(2024)]%
        {resnet}
\bibfield{author}{\bibinfo{person}{Yury Gorishniy}, \bibinfo{person}{Ivan
  Rubachev}, \bibinfo{person}{Valentin Khrulkov}, {and} \bibinfo{person}{Artem
  Babenko}.} \bibinfo{year}{2024}\natexlab{}.
\newblock \showarticletitle{Revisiting deep learning models for tabular data}.
  In \bibinfo{booktitle}{\emph{NeurIPS}}.
\newblock


\bibitem[Grinsztajn et~al\mbox{.}(2022)]%
        {grinsztajn2022tree}
\bibfield{author}{\bibinfo{person}{L{\'e}o Grinsztajn},
  \bibinfo{person}{Edouard Oyallon}, {and} \bibinfo{person}{Ga{\"e}l
  Varoquaux}.} \bibinfo{year}{2022}\natexlab{}.
\newblock \showarticletitle{Why do tree-based models still outperform deep
  learning on typical tabular data?}
\newblock \bibinfo{journal}{\emph{Advances in neural information processing
  systems}} (\bibinfo{year}{2022}).
\newblock


\bibitem[Hyafil and Rivest(1976)]%
        {npcomplete}
\bibfield{author}{\bibinfo{person}{Laurent Hyafil} {and}
  \bibinfo{person}{Ronald~L. Rivest}.} \bibinfo{year}{1976}\natexlab{}.
\newblock \showarticletitle{Constructing optimal binary decision trees is
  NP-complete}.
\newblock \bibinfo{journal}{\emph{Inform. Process. Lett.}}
  (\bibinfo{year}{1976}).
\newblock


\bibitem[Kairgeldin and Carreira-Perpi\~{n}\'{a}n(2024)]%
        {10.1145/3637528.3671903}
\bibfield{author}{\bibinfo{person}{Rasul Kairgeldin} {and}
  \bibinfo{person}{Miguel~\'{A}. Carreira-Perpi\~{n}\'{a}n}.}
  \bibinfo{year}{2024}\natexlab{}.
\newblock \showarticletitle{Bivariate Decision Trees: Smaller, Interpretable,
  More Accurate}. In \bibinfo{booktitle}{\emph{KDD}}.
\newblock


\bibitem[Ke et~al\mbox{.}(2017)]%
        {ke2017lightgbm}
\bibfield{author}{\bibinfo{person}{Guolin Ke}, \bibinfo{person}{Qi Meng},
  \bibinfo{person}{Thomas Finley}, \bibinfo{person}{Taifeng Wang},
  \bibinfo{person}{Wei Chen}, \bibinfo{person}{Weidong Ma},
  \bibinfo{person}{Qiwei Ye}, {and} \bibinfo{person}{Tie-Yan Liu}.}
  \bibinfo{year}{2017}\natexlab{}.
\newblock \showarticletitle{Lightgbm: A highly efficient gradient boosting
  decision tree}.
\newblock \bibinfo{journal}{\emph{NeurIPS}} (\bibinfo{year}{2017}).
\newblock


\bibitem[Kohler et~al\mbox{.}(2024)]%
        {kohler2024interpretable}
\bibfield{author}{\bibinfo{person}{Hector Kohler}, \bibinfo{person}{Quentin
  Delfosse}, \bibinfo{person}{Riad Akrour}, \bibinfo{person}{Kristian
  Kersting}, {and} \bibinfo{person}{Philippe Preux}.}
  \bibinfo{year}{2024}\natexlab{}.
\newblock \showarticletitle{Interpretable and Editable Programmatic Tree
  Policies for Reinforcement Learning}. In
  \bibinfo{booktitle}{\emph{Seventeenth European Workshop on Reinforcement
  Learning}}.
\newblock


\bibitem[L.~Puterman(1994)]%
        {puterman}
\bibfield{editor}{\bibinfo{person}{Martin L.~Puterman}} (Ed.).
  \bibinfo{year}{1994}\natexlab{}.
\newblock \bibinfo{booktitle}{\emph{Markov Decision Processes: Discrete
  Stochastic Dynamic Programming}}.
\newblock \bibinfo{publisher}{John Wiley \& Sons}.
\newblock


\bibitem[Lin et~al\mbox{.}(2020)]%
        {lin2020generalized}
\bibfield{author}{\bibinfo{person}{Jimmy Lin}, \bibinfo{person}{Chudi Zhong},
  \bibinfo{person}{Diane Hu}, \bibinfo{person}{Cynthia Rudin}, {and}
  \bibinfo{person}{Margo Seltzer}.} \bibinfo{year}{2020}\natexlab{}.
\newblock \showarticletitle{Generalized and scalable optimal sparse decision
  trees}. In \bibinfo{booktitle}{\emph{ICML}}.
\newblock


\bibitem[Lipton(2018)]%
        {lipton}
\bibfield{author}{\bibinfo{person}{Zachary~C. Lipton}.}
  \bibinfo{year}{2018}\natexlab{}.
\newblock \showarticletitle{The Mythos of Model Interpretability: In machine
  learning, the concept of interpretability is both important and slippery.}
\newblock \bibinfo{journal}{\emph{Queue}} (\bibinfo{year}{2018}).
\newblock


\bibitem[Loh(2014)]%
        {loh2014fifty}
\bibfield{author}{\bibinfo{person}{Wei-Yin Loh}.}
  \bibinfo{year}{2014}\natexlab{}.
\newblock \showarticletitle{Fifty years of classification and regression
  trees}.
\newblock \bibinfo{journal}{\emph{International Statistical Review}}
  (\bibinfo{year}{2014}).
\newblock


\bibitem[Marton et~al\mbox{.}(2024)]%
        {marton2024sympolsymbolictreebasedonpolicy}
\bibfield{author}{\bibinfo{person}{Sascha Marton}, \bibinfo{person}{Tim Grams},
  \bibinfo{person}{Florian Vogt}, \bibinfo{person}{Stefan Lüdtke},
  \bibinfo{person}{Christian Bartelt}, {and} \bibinfo{person}{Heiner
  Stuckenschmidt}.} \bibinfo{year}{2024}\natexlab{}.
\newblock \showarticletitle{Mitigating Information Loss in Tree-Based
  Reinforcement Learning via Direct Optimization}. In
  \bibinfo{booktitle}{\emph{ICLR}}.
\newblock


\bibitem[Mazumder et~al\mbox{.}(2022)]%
        {quantbnb}
\bibfield{author}{\bibinfo{person}{Rahul Mazumder}, \bibinfo{person}{Xiang
  Meng}, {and} \bibinfo{person}{Haoyue Wang}.} \bibinfo{year}{2022}\natexlab{}.
\newblock \showarticletitle{Quant-{B}n{B}: A Scalable Branch-and-Bound Method
  for Optimal Decision Trees with Continuous Features}. In
  \bibinfo{booktitle}{\emph{ICML}}.
\newblock


\bibitem[McTavish et~al\mbox{.}(2022)]%
        {McTavish_Zhong_Achermann_Karimalis_Chen_Rudin_Seltzer_2022}
\bibfield{author}{\bibinfo{person}{Hayden McTavish}, \bibinfo{person}{Chudi
  Zhong}, \bibinfo{person}{Reto Achermann}, \bibinfo{person}{Ilias Karimalis},
  \bibinfo{person}{Jacques Chen}, \bibinfo{person}{Cynthia Rudin}, {and}
  \bibinfo{person}{Margo Seltzer}.} \bibinfo{year}{2022}\natexlab{}.
\newblock \showarticletitle{Fast Sparse Decision Tree Optimization via
  Reference Ensembles}.
\newblock \bibinfo{journal}{\emph{AAAI}} (\bibinfo{year}{2022}).
\newblock


\bibitem[Mohamed et~al\mbox{.}(2012)]%
        {pruning2}
\bibfield{author}{\bibinfo{person}{W~Nor Haizan~W Mohamed},
  \bibinfo{person}{Mohd Najib~Mohd Salleh}, {and} \bibinfo{person}{Abdul~Halim
  Omar}.} \bibinfo{year}{2012}\natexlab{}.
\newblock \showarticletitle{A comparative study of reduced error pruning method
  in decision tree algorithms}. In \bibinfo{booktitle}{\emph{2012 IEEE
  International conference on control system, computing and engineering}}.
\newblock


\bibitem[Murthy and Salzberg(1995a)]%
        {how-eff}
\bibfield{author}{\bibinfo{person}{Sreerama Murthy} {and}
  \bibinfo{person}{Steven Salzberg}.} \bibinfo{year}{1995}\natexlab{a}.
\newblock \showarticletitle{Decision tree induction: how effective is the
  greedy heuristic?}. In \bibinfo{booktitle}{\emph{KDD}}.
\newblock


\bibitem[Murthy and Salzberg(1995b)]%
        {Murthy}
\bibfield{author}{\bibinfo{person}{Sreerama Murthy} {and}
  \bibinfo{person}{Steven Salzberg}.} \bibinfo{year}{1995}\natexlab{b}.
\newblock \showarticletitle{Lookahead and Pathology in Decision Tree
  Induction}. In \bibinfo{booktitle}{\emph{IJCAI}}.
\newblock


\bibitem[Murthy et~al\mbox{.}(1994)]%
        {murthy1994system}
\bibfield{author}{\bibinfo{person}{Sreerama~K Murthy}, \bibinfo{person}{Simon
  Kasif}, {and} \bibinfo{person}{Steven Salzberg}.}
  \bibinfo{year}{1994}\natexlab{}.
\newblock \showarticletitle{A system for induction of oblique decision trees}.
\newblock \bibinfo{journal}{\emph{Journal of artificial intelligence research}}
  (\bibinfo{year}{1994}).
\newblock


\bibitem[Norouzi et~al\mbox{.}(2015)]%
        {NIPS2015_1579779b}
\bibfield{author}{\bibinfo{person}{Mohammad Norouzi}, \bibinfo{person}{Maxwell
  Collins}, \bibinfo{person}{Matthew~A Johnson}, \bibinfo{person}{David~J
  Fleet}, {and} \bibinfo{person}{Pushmeet Kohli}.}
  \bibinfo{year}{2015}\natexlab{}.
\newblock \showarticletitle{Efficient Non-greedy Optimization of Decision
  Trees}. In \bibinfo{booktitle}{\emph{NeurIPS}}.
\newblock


\bibitem[Norton(1989)]%
        {norton}
\bibfield{author}{\bibinfo{person}{Steven~W. Norton}.}
  \bibinfo{year}{1989}\natexlab{}.
\newblock \showarticletitle{Generating Better Decision Trees}. In
  \bibinfo{booktitle}{\emph{IJCAI}}.
\newblock


\bibitem[Pedregosa et~al\mbox{.}(2011)]%
        {scikit-learn}
\bibfield{author}{\bibinfo{person}{F. Pedregosa}, \bibinfo{person}{G.
  Varoquaux}, \bibinfo{person}{A. Gramfort}, \bibinfo{person}{V. Michel},
  \bibinfo{person}{B. Thirion}, \bibinfo{person}{O. Grisel},
  \bibinfo{person}{M. Blondel}, \bibinfo{person}{P. Prettenhofer},
  \bibinfo{person}{R. Weiss}, \bibinfo{person}{V. Dubourg}, \bibinfo{person}{J.
  Vanderplas}, \bibinfo{person}{A. Passos}, \bibinfo{person}{D. Cournapeau},
  \bibinfo{person}{M. Brucher}, \bibinfo{person}{M. Perrot}, {and}
  \bibinfo{person}{E. Duchesnay}.} \bibinfo{year}{2011}\natexlab{}.
\newblock \showarticletitle{Scikit-learn: Machine Learning in {P}ython}.
\newblock \bibinfo{journal}{\emph{Journal of Machine Learning Research}}
  (\bibinfo{year}{2011}).
\newblock


\bibitem[Prokhorenkova et~al\mbox{.}(2018)]%
        {10.5555/3327757.3327770}
\bibfield{author}{\bibinfo{person}{Liudmila Prokhorenkova},
  \bibinfo{person}{Gleb Gusev}, \bibinfo{person}{Aleksandr Vorobev},
  \bibinfo{person}{Anna~Veronika Dorogush}, {and} \bibinfo{person}{Andrey
  Gulin}.} \bibinfo{year}{2018}\natexlab{}.
\newblock \showarticletitle{CatBoost: unbiased boosting with categorical
  features}. In \bibinfo{booktitle}{\emph{NeurIPS}}.
\newblock


\bibitem[Quinlan(1986)]%
        {ID3}
\bibfield{author}{\bibinfo{person}{J.~R. Quinlan}.}
  \bibinfo{year}{1986}\natexlab{}.
\newblock \showarticletitle{Induction of Decision Trees}.
\newblock \bibinfo{journal}{\emph{Machine Learning}} (\bibinfo{year}{1986}).
\newblock


\bibitem[Quinlan(1993)]%
        {c45}
\bibfield{author}{\bibinfo{person}{J~Ross Quinlan}.}
  \bibinfo{year}{1993}\natexlab{}.
\newblock \showarticletitle{C4. 5: Programs for machine learning}.
\newblock \bibinfo{journal}{\emph{Elsevier}} (\bibinfo{year}{1993}).
\newblock


\bibitem[Raffin et~al\mbox{.}(2021)]%
        {stable-baselines3}
\bibfield{author}{\bibinfo{person}{Antonin Raffin}, \bibinfo{person}{Ashley
  Hill}, \bibinfo{person}{Adam Gleave}, \bibinfo{person}{Anssi Kanervisto},
  \bibinfo{person}{Maximilian Ernestus}, {and} \bibinfo{person}{Noah Dormann}.}
  \bibinfo{year}{2021}\natexlab{}.
\newblock \showarticletitle{Stable-Baselines3: Reliable Reinforcement Learning
  Implementations}.
\newblock \bibinfo{journal}{\emph{Journal of Machine Learning Research}}
  (\bibinfo{year}{2021}).
\newblock


\bibitem[Saux et~al\mbox{.}(2023)]%
        {saux:hal-04192198}
\bibfield{author}{\bibinfo{person}{Patrick Saux}, \bibinfo{person}{Pierre
  Bauvin}, \bibinfo{person}{Violeta Raverdy}, \bibinfo{person}{Julien Teigny},
  \bibinfo{person}{H{\'e}l{\`e}ne Verkindt}, \bibinfo{person}{Tomy
  Soumphonphakdy}, \bibinfo{person}{Maxence Debert}, \bibinfo{person}{Anne
  Jacobs}, \bibinfo{person}{Daan Jacobs}, \bibinfo{person}{Valerie Monpellier},
  \bibinfo{person}{Phong~Ching Lee}, \bibinfo{person}{Chin~Hong Lim},
  \bibinfo{person}{Johanna~C Andersson-Assarsson}, \bibinfo{person}{Lena
  Carlsson}, \bibinfo{person}{Per-Arne Svensson}, \bibinfo{person}{Florence
  Galtier}, \bibinfo{person}{Guelareh Dezfoulian}, \bibinfo{person}{Mihaela
  Moldovanu}, \bibinfo{person}{Severine Andrieux}, \bibinfo{person}{Julien
  Couster}, \bibinfo{person}{Marie Lepage}, \bibinfo{person}{Erminia Lembo},
  \bibinfo{person}{Ornella Verrastro}, \bibinfo{person}{Maud Robert},
  \bibinfo{person}{Paulina Salminen}, \bibinfo{person}{Geltrude Mingrone},
  \bibinfo{person}{Ralph Peterli}, \bibinfo{person}{Ricardo~V Cohen},
  \bibinfo{person}{Carlos Zerrweck}, \bibinfo{person}{David Nocca},
  \bibinfo{person}{Carel~W Le~Roux}, \bibinfo{person}{Robert Caiazzo},
  \bibinfo{person}{Philippe Preux}, {and} \bibinfo{person}{Fran{\c c}ois
  Pattou}.} \bibinfo{year}{2023}\natexlab{}.
\newblock \showarticletitle{{Development and validation of an interpretable
  machine learning-based calculator for predicting 5-year weight trajectories
  after bariatric surgery: a multinational retrospective cohort SOPHIA study}}.
\newblock \bibinfo{journal}{\emph{{The Lancet Digital Health}}}
  (\bibinfo{year}{2023}).
\newblock


\bibitem[Somepalli et~al\mbox{.}(2021)]%
        {somepalli2021saintimprovedneuralnetworks}
\bibfield{author}{\bibinfo{person}{Gowthami Somepalli}, \bibinfo{person}{Micah
  Goldblum}, \bibinfo{person}{Avi Schwarzschild}, \bibinfo{person}{C.~Bayan
  Bruss}, {and} \bibinfo{person}{Tom Goldstein}.}
  \bibinfo{year}{2021}\natexlab{}.
\newblock \bibinfo{title}{SAINT: Improved Neural Networks for Tabular Data via
  Row Attention and Contrastive Pre-Training}.
\newblock
\showeprint[arxiv]{2106.01342}~[cs.LG]


\bibitem[Topin et~al\mbox{.}(2021)]%
        {topin2021iterative}
\bibfield{author}{\bibinfo{person}{Nicholay Topin}, \bibinfo{person}{Stephanie
  Milani}, \bibinfo{person}{Fei Fang}, {and} \bibinfo{person}{Manuela Veloso}.}
  \bibinfo{year}{2021}\natexlab{}.
\newblock \showarticletitle{Iterative bounding mdps: Learning interpretable
  policies via non-interpretable methods}. In \bibinfo{booktitle}{\emph{AAAI}}.
\newblock


\bibitem[van~der Linden et~al\mbox{.}(2023)]%
        {pystreed}
\bibfield{author}{\bibinfo{person}{Jacobus van~der Linden},
  \bibinfo{person}{Mathijs de Weerdt}, {and} \bibinfo{person}{Emir
  Demirovi\'{c}}.} \bibinfo{year}{2023}\natexlab{}.
\newblock \showarticletitle{Necessary and Sufficient Conditions for Optimal
  Decision Trees using Dynamic Programming}. In
  \bibinfo{booktitle}{\emph{NeurIPS}}.
\newblock


\bibitem[van~der Linden et~al\mbox{.}(2024)]%
        {vanderlinden2024optimalgreedydecisiontrees}
\bibfield{author}{\bibinfo{person}{Jacobus G.~M. van~der Linden},
  \bibinfo{person}{Daniël Vos}, \bibinfo{person}{Mathijs~M. de Weerdt},
  \bibinfo{person}{Sicco Verwer}, {and} \bibinfo{person}{Emir Demirović}.}
  \bibinfo{year}{2024}\natexlab{}.
\newblock \bibinfo{title}{Optimal or Greedy Decision Trees? Revisiting their
  Objectives, Tuning, and Performance}.
\newblock
\showeprint[arxiv]{2409.12788}~[cs.LG]


\bibitem[Vanschoren et~al\mbox{.}(2014)]%
        {10.1145/2641190.2641198}
\bibfield{author}{\bibinfo{person}{Joaquin Vanschoren}, \bibinfo{person}{Jan~N.
  van Rijn}, \bibinfo{person}{Bernd Bischl}, {and} \bibinfo{person}{Luis
  Torgo}.} \bibinfo{year}{2014}\natexlab{}.
\newblock \showarticletitle{OpenML: networked science in machine learning}.
\newblock \bibinfo{journal}{\emph{KDD}}.
\newblock


\bibitem[Verwer and Zhang(2017)]%
        {verwer2017learning}
\bibfield{author}{\bibinfo{person}{Sicco Verwer} {and}
  \bibinfo{person}{Yingqian Zhang}.} \bibinfo{year}{2017}\natexlab{}.
\newblock \showarticletitle{Learning decision trees with flexible constraints
  and objectives using integer optimization}. In
  \bibinfo{booktitle}{\emph{Integration of AI and OR Techniques in Constraint
  Programming: 14th International Conference}}.
\newblock


\bibitem[Verwer and Zhang(2019)]%
        {binoct}
\bibfield{author}{\bibinfo{person}{Sicco Verwer} {and}
  \bibinfo{person}{Yingqian Zhang}.} \bibinfo{year}{2019}\natexlab{}.
\newblock \showarticletitle{Learning optimal classification trees using a
  binary linear program formulation}. In \bibinfo{booktitle}{\emph{AAAI}}.
\newblock


\bibitem[Vos and Verwer(2023)]%
        {mdpdt}
\bibfield{author}{\bibinfo{person}{Dani\"{e}l Vos} {and} \bibinfo{person}{Sicco
  Verwer}.} \bibinfo{year}{2023}\natexlab{}.
\newblock \showarticletitle{Optimal decision tree policies for Markov decision
  processes}. In \bibinfo{booktitle}{\emph{IJCAI}}.
\newblock


\bibitem[Vos and Verwer(2024)]%
        {vos2024optimizinginterpretabledecisiontree}
\bibfield{author}{\bibinfo{person}{Daniël Vos} {and} \bibinfo{person}{Sicco
  Verwer}.} \bibinfo{year}{2024}\natexlab{}.
\newblock \bibinfo{title}{Optimizing Interpretable Decision Tree Policies for
  Reinforcement Learning}.
\newblock
\showeprint[arxiv]{2408.11632}~[cs.LG]


\bibitem[Whiteson(2014)]%
        {higgs_280}
\bibfield{author}{\bibinfo{person}{Daniel Whiteson}.}
  \bibinfo{year}{2014}\natexlab{}.
\newblock \bibinfo{title}{{HIGGS}}.
\newblock \bibinfo{howpublished}{UCI Machine Learning Repository}.
\newblock


\bibitem[Zharmagambetov et~al\mbox{.}(2021a)]%
        {9534446}
\bibfield{author}{\bibinfo{person}{Arman Zharmagambetov},
  \bibinfo{person}{Magzhan Gabidolla}, {and} \bibinfo{person}{Miguel~Ê.
  Carreira-Perpiñán}.} \bibinfo{year}{2021}\natexlab{a}.
\newblock \showarticletitle{Improved Boosted Regression Forests Through
  Non-Greedy Tree Optimization}. In \bibinfo{booktitle}{\emph{IJCNN}}.
\newblock


\bibitem[Zharmagambetov et~al\mbox{.}(2021b)]%
        {9533597}
\bibfield{author}{\bibinfo{person}{Arman Zharmagambetov},
  \bibinfo{person}{Suryabhan~Singh Hada}, \bibinfo{person}{Magzhan Gabidolla},
  {and} \bibinfo{person}{Miguel~Á. Carreira-Perpiñán}.}
  \bibinfo{year}{2021}\natexlab{b}.
\newblock \showarticletitle{Non-Greedy Algorithms for Decision Tree
  Optimization: An Experimental Comparison}. In
  \bibinfo{booktitle}{\emph{IJCNN}}.
\newblock


\end{thebibliography}

\appendix
\section{Proof of Propostion~\ref{prop:equiv}}
\label{sec:proof-equiv}
For the purpose of the proof, we overload the definition of $J_\alpha$ and $\mathcal L_\alpha$, to make explicit the dependency on the dataset and the maximum depth. 
As such, $J_\alpha(\pi)$ becomes $J_\alpha(\pi, {\mathcal E}, D)$ and ${\mathcal L}_\alpha(T)$ becomes ${\mathcal L}_\alpha(T, {\mathcal E})$. 
Let us first show that the relation $J_\alpha(\pi, {\mathcal E}, 0) = -{\mathcal L}_\alpha(T, {\mathcal E})$ is true. 
If the maximum depth is $D = 0$ then $\pi(s_0)$ is necessarily a class assignment, in which case the expected number of splits is zero and the relation is obviously true since the reward is the opposite of the average classification loss. 
Now assume it is true for any dataset and tree of depth at most $D$ with $D \geq 0$ and let us prove that it holds for all trees of depth $D + 1$. 
For a tree $T$ of depth $D + 1$ the root is necessarily a split node. Let $T_l = E(\pi, s_l)$ and $T_r = E(\pi, s_r)$ be the left and right sub-trees of the root node of $T$. 
Since both sub-trees are of depth at most $D$, the relation holds and we have $J_\alpha(\pi, X_l, D) = {\mathcal L}_\alpha(T_l, X_l)$ and $J_\alpha(\pi, X_r, D) = {\mathcal L}_\alpha(T_r, X_r)$, where $X_l$ and $X_r$ are the datasets of the ``right'' and ``left'' states to which the MDP transitions---with probabilities $p_l$ and $p_r$---upon application of $\pi(s_0)$ in $s_0$, as described in the MDP formulation. 
Moreover, from the definition of the policy return we have 
\begin{align*}
   J_\alpha(&\pi, {\mathcal E}, D + 1) = -\alpha + p_l * J_\alpha(\pi, X_l, D) + p_r * J_\alpha(\pi, X_r, D)\\
   &= -\alpha - p_l * {\mathcal L}_\alpha(T_l, X_l) - p_r * {\mathcal L}_\alpha(T_r, D)\\
   &= -\alpha - p_l * \Bigg(\frac{1}{|X_l|}\sum_{(x_i, y_i)\in X_l}\ell(y_i, T_l(x_i))  + \alpha C(T_l)\Bigg)\\
   &\ - p_r * \Bigg(\frac{1}{|X_r|}\sum_{(x_i, y_i)\in X_r}\ell(y_i, T_r(x_i))  + \alpha C(T_r)\Bigg)\\
   &= -\frac{1}{N}\sum_{(x_i, y_i)\in X}\ell(y_i, T(x_i)) - \alpha (1 + p_l C(T_l) + p_r C(T_r))\\
   &= -{\mathcal L}(T, {\mathcal E}) 
\end{align*}
\section{Additional Experiments and Hyperparameters}
In this section we provide additional experimental results. In Table \ref{tab:regression}, we compare DPDT trees to CART and STreeD trees using 50 train/test splits of regression datasets from \citep{grinsztajn2022tree}. All algorithms are run with default hyperparameters. The configuration of DPDT is (4, 4, 4) or (4,4,4,4,4). STreeD is run with a time limit of 4 hours per tree computation and on binarized versions of the datasets. Both for depth-3 and depth-5 trees, DPDT outperforms other baselines in terms of train and test accuracies. Indeed, because STreeD runs on ``approximated'' datasets, it performs pretty poorly. 

In Table \ref{tab:more-res}, we compare DPDT(4, 4, 4, 4, 4) to additional optimal decision tree baselines on datasets with \textbf{binary features}. The optimal decision tree baselines run with default hyperparameters and a time-limit of 10 minutes. The results show that even on binary datasets that optimal algorithms are designed to handle well; DPDT outperforms other baselimes. This is likely because optimal trees are slow and/or don't scale well to depth 5.

In Table~\ref{tab:lookahead} compare DPDT to lookahead depth-3 trees when optimizing Eq.\ref{eq:suplearning}. Unlike the other greedy approaches, lookahead decision trees~\citep{norton} do not pick the split that optimizes a heuristic immediately. Instead, they pick a split that sets up the best possible heuristic value on the following split. Lookahead-1 chooses nodes at depth $d<3$ by looking 1 depth in the future: it looks for the sequence of 2 splits that maximizes the information gain at depth $d + 1$. Lookahead-2 is the optimal depth-3 tree and Lookahead-0 would be just building the tree greedily like CART. The conclusion are roughly the same as for Table~\ref{tab:tree_comparison_combined}. Both lookahead trees and DPDT\footnote{\url{https://github.com/KohlerHECTOR/DPDTreeEstimator}} are in \texttt{Python} which makes them slow but comparable.

We also provide the hyperparameters necessary to reproduce experiments from section \ref{sec:generalization} and \ref{sec:boosting} in Table~\ref{tab:tree_hyperparams}.

\begin{table}
\caption{Train accuracies of depth-3 trees (with number of operations). Lookahead trees are trained with a time limit of 12 hours.}
\centering
\begin{tabular}{l|c|c}
\toprule
\textbf{Dataset} & \textbf{DPDT} & \textbf{Lookahead-1} \\
\midrule
avila & 57.22 (1304) & OoT \\
bank & 97.99 (699) & 96.54 (7514) \\
bean & 85.30 (1297) & OoT \\
bidding & 99.27 (744) & 98.12 (20303) \\
eeg & 69.38 (1316) & 69.09 (10108) \\
fault & 67.40 (1263) & 67.20 (32514) \\
htru & 98.01 (1388) & OoT \\
magic & 82.81 (1451) & OoT \\
occupancy & 99.31 (1123) & 99.01 (15998) \\
page & 97.03 (1243) & 96.44 (16295) \\
raisin & 88.61 (1193) & 86.94 (9843) \\
rice & 93.44 (1367) & 93.24 (37766) \\
room & 99.23 (1196) & 99.04 (5638) \\
segment & 87.88 (871) & 68.83 (24833) \\
skin & 96.60 (1300) & 96.61 (1290) \\
wilt & 99.47 (862) & 99.31 (36789) \\
\bottomrule
\end{tabular}
\label{tab:lookahead}
\end{table}

\begin{table*}
\centering
\footnotesize
\caption{Mean train and test scores (with standard errors) for regression datasets over 50 cross-validation runs.}\label{tab:regression}
\begin{tabular}{l|rr|rr|rr||rr|rr|rr|rr}
\hline
 & \multicolumn{6}{c||}{\textbf{Depth 3}} & \multicolumn{6}{c|}{\textbf{Depth 5}} \\ \hline
Dataset & \multicolumn{2}{c|}{DPDT} & \multicolumn{2}{c|}{Optimal} & \multicolumn{2}{c||}{CART} & \multicolumn{2}{c|}{DPDT} & \multicolumn{2}{c|}{Optimal} & \multicolumn{2}{c|}{CART} \\
 & Train & Test & Train & Test & Train & Test & Train & Test & Train & Test & Train & Test \\
\hline
nyc-taxi & 39.0 ± 0.0 & 38.9 ± 0.2 & 33.8 ± 0.0 & 33.8 ± 0.1 & 39.0 ± 0.0 & 38.9 ± 0.2 & 45.8 ± 0.0 & 45.7 ± 0.2 & 33.8 ± 0.0 & 33.8 ± 0.1 & 42.7 ± 0.0 & 42.6 ± 0.2 \\
medical\_charges & 95.2 ± 0.0 & 95.2 ± 0.0 & 90.1 ± 0.0 & 90.1 ± 0.1 & 95.2 ± 0.0 & 95.2 ± 0.0 & 97.7 ± 0.0 & 97.7 ± 0.0 & 90.1 ± 0.0 & 90.1 ± 0.1 & 97.7 ± 0.0 & 97.7 ± 0.0 \\
diamonds & 93.0 ± 0.0 & 92.9 ± 0.1 & 90.1 ± 0.0 & 90.1 ± 0.1 & 92.7 ± 0.0 & 92.6 ± 0.1 & 94.2 ± 0.0 & 94.0 ± 0.1 & 90.1 ± 0.0 & 90.1 ± 0.1 & 94.1 ± 0.0 & 93.9 ± 0.1 \\
house\_16H & 39.9 ± 0.1 & 38.1 ± 2.5 & 32.8 ± 0.1 & 29.4 ± 1.6 & 35.8 ± 0.1 & 35.8 ± 1.9 & 59.4 ± 0.1 & 35.2 ± 4.1 & 32.8 ± 0.1 & 29.4 ± 1.6 & 51.5 ± 0.1 & 41.3 ± 3.1 \\
house\_sales & 67.0 ± 0.0 & 66.0 ± 0.4 & 65.1 ± 0.0 & 64.4 ± 0.4 & 66.8 ± 0.0 & 66.1 ± 0.4 & 77.6 ± 0.0 & 76.1 ± 0.3 & 65.1 ± 0.0 & 64.4 ± 0.4 & 76.8 ± 0.0 & 75.3 ± 0.4 \\
superconduct & 73.1 ± 0.0 & 72.7 ± 0.5 & 70.9 ± 0.0 & 70.5 ± 0.5 & 70.4 ± 0.0 & 69.7 ± 0.5 & 83.0 ± 0.0 & 81.7 ± 0.4 & 70.9 ± 0.0 & 70.5 ± 0.5 & 78.2 ± 0.0 & 76.5 ± 0.5 \\
houses & 51.7 ± 0.0 & 50.7 ± 0.7 & 48.5 ± 0.1 & 47.3 ± 0.7 & 49.5 ± 0.0 & 48.4 ± 0.7 & 69.1 ± 0.0 & 67.6 ± 0.5 & 48.5 ± 0.1 & 47.3 ± 0.7 & 60.4 ± 0.1 & 58.5 ± 0.6 \\
Bike\_Sharing & 55.2 ± 0.0 & 54.7 ± 0.5 & 45.1 ± 0.1 & 44.8 ± 0.7 & 48.1 ± 0.0 & 47.9 ± 0.5 & 65.2 ± 0.0 & 63.3 ± 0.5 & 45.1 ± 0.1 & 44.8 ± 0.7 & 59.1 ± 0.0 & 58.6 ± 0.5 \\
elevators & 48.0 ± 0.0 & 46.8 ± 1.1 & 40.2 ± 0.1 & 38.2 ± 1.0 & 46.8 ± 0.0 & 45.5 ± 1.2 & 65.6 ± 0.0 & 61.2 ± 1.0 & 40.2 ± 0.1 & 38.2 ± 1.0 & 61.9 ± 0.0 & 58.0 ± 1.2 \\
pol & 72.2 ± 0.0 & 71.3 ± 0.6 & 67.8 ± 0.1 & 67.5 ± 0.9 & 72.0 ± 0.0 & 71.2 ± 0.8 & 93.3 ± 0.0 & 92.4 ± 0.3 & 67.8 ± 0.1 & 67.5 ± 0.9 & 92.1 ± 0.0 & 91.7 ± 0.3 \\
MiamiHousing2016 & 62.3 ± 0.0 & 60.4 ± 0.8 & 60.8 ± 0.0 & 58.3 ± 0.8 & 62.3 ± 0.0 & 60.6 ± 0.8 & 79.8 ± 0.0 & 77.5 ± 0.5 & 60.8 ± 0.0 & 58.3 ± 0.8 & 77.3 ± 0.1 & 74.7 ± 0.6 \\
Ailerons & 63.5 ± 0.0 & 62.6 ± 0.7 & 61.6 ± 0.0 & 60.3 ± 0.7 & 63.5 ± 0.0 & 62.6 ± 0.7 & 76.0 ± 0.0 & 72.9 ± 0.6 & 61.6 ± 0.0 & 60.3 ± 0.7 & 75.5 ± 0.0 & 73.2 ± 0.5 \\
Brazilian\_houses & 90.7 ± 0.0 & 90.3 ± 0.8 & 89.2 ± 0.0 & 89.4 ± 0.8 & 90.7 ± 0.0 & 90.4 ± 0.8 & 97.6 ± 0.0 & 96.6 ± 0.4 & 89.2 ± 0.0 & 89.4 ± 0.8 & 97.3 ± 0.0 & 96.4 ± 0.4 \\
sulfur & 72.5 ± 0.1 & 66.6 ± 2.2 & 35.7 ± 0.1 & 19.1 ± 6.7 & 72.0 ± 0.1 & 68.0 ± 2.2 & 89.0 ± 0.0 & 68.4 ± 6.7 & 35.7 ± 0.1 & 19.1 ± 6.7 & 81.8 ± 0.1 & 74.0 ± 2.2 \\
yprop\_41 & 6.3 ± 0.0 & 2.3 ± 0.7 & 3.6 ± 0.0 & 1.5 ± 0.4 & 6.2 ± 0.0 & 2.1 ± 0.8 & 13.2 ± 0.0 & 1.2 ± 1.7 & 3.6 ± 0.0 & 1.5 ± 0.4 & 10.8 ± 0.0 & -1.7 ± 1.4 \\
cpu\_act & 93.4 ± 0.0 & 92.0 ± 0.6 & 89.0 ± 0.0 & 86.5 ± 1.9 & 93.4 ± 0.0 & 92.0 ± 0.6 & 96.5 ± 0.0 & 94.7 ± 0.5 & 89.0 ± 0.0 & 86.5 ± 1.9 & 96.3 ± 0.0 & 95.1 ± 0.4 \\
wine\_quality & 27.9 ± 0.0 & 23.3 ± 0.9 & 25.2 ± 0.0 & 23.7 ± 0.8 & 27.7 ± 0.0 & 24.5 ± 0.8 & 37.4 ± 0.0 & 26.7 ± 1.0 & 25.2 ± 0.0 & 23.7 ± 0.8 & 35.9 ± 0.1 & 26.7 ± 0.9 \\
abalone & 46.3 ± 0.0 & 39.6 ± 1.6 & 42.5 ± 0.0 & 40.4 ± 1.4 & 43.3 ± 0.0 & 39.2 ± 1.2 & 58.6 ± 0.0 & 44.7 ± 1.8 & 42.5 ± 0.0 & 40.4 ± 1.4 & 54.5 ± 0.0 & 46.3 ± 1.4 \\
\hline
\end{tabular}
\end{table*}
\begin{table*}
    \centering
    \footnotesize
    \caption{Train/test accuracies of different decision tree induction algorithms. All algorithms induce trees of depth at most 5 on 8 classification datasets. A time limit of 10 minutes is set for OCT-type algorithms. The values in this table are averaged over 3 seeds giving 3 different train/test datasets. 
    }
    \begin{tabular}{|c|ccc|ccccc|ccccc|}
    \hline \multicolumn{4}{|c}{ Datasets } & \multicolumn{5}{|c|}{ Train Accuracy depth-5}& \multicolumn{5}{c|}{ Test Accuracy depth-5} \\
    Names & Samples & Features & Classes & DPDT & OCT & MFOCT & BinOCT & CART & DPDT & OCT & MFOCT & BinOCT & CART \\
    \hline 
    balance-scale & $624$ & 4 & 3 & $ 90.9 \%$ & $ 71.8 \%$ & $ 82.6 \%$ & $ 67.5 \%$ & $ 86.5 \%$ & $ 77.1 \%$ & $ 66.9 \%$ & $ 71.3 \%$ & $ 61.6 \%$ & $ 76.4 \%$  \\
    breast-cancer & $276$ & 9 & 2 & $ 94.2 \%$ & $ 88.6 \%$ & $ 91.1 \%$ & $ 75.4 \%$ & $ 87.9 \%$ & $ 66.4 \%$ & $ 67.1 \%$ & $ 73.8 \%$ & $ 62.4 \%$ & $ 70.3 \%$ \\
    car-evaluation & $1728$ & 6 & 4 & $ 92.2 \%$  & $ 70.1 \%$ & $ 80.4 \%$ & $ 84.0 \%$ & $ 87.1 \%$ & $ 90.3 \%$  &$ 69.5 \%$ & $ 79.8 \%$ & $ 82.3 \%$ & $ 87.1 \%$ \\
    hayes-roth & $160$ & 9 & 3 & $ 93.3 \%$  & $ 82.9 \%$ & $ 95.4 \%$ & $ 64.6 \%$ & $ 76.7 \%$ & $ 75.4 \%$ &$ 77.5 \%$ & $ 77.5 \%$ & $ 54.2 \%$ & $ 69.2 \%$ \\
    house-votes-84 & $232$ & 16 & 2 & $ 100.0 \%$ & $ 100.0 \%$ & $ 100.0 \%$ & $ 100.0 \%$ & $ 99.4 \%$ & $ 95.4 \%$ &$ 93.7 \%$ & $ 94.3 \%$ & $ 96.0 \%$ & $ 95.1 \%$\\
    soybean-small & $46$ & 50 & 4 & $ 100.0 \%$  & $ 100.0 \%$ & $ 100.0 \%$ & $ 76.8 \%$ & $ 100.0 \%$ & $ 93.1 \%$ &$ 94.4 \%$ & $ 91.7 \%$ & $ 72.2 \%$ & $ 93.1 \%$ \\
    spect & $266$ & 22 & 2 & $ 93.0 \%$ & $ 92.5 \%$ & $ 93.0 \%$ & $ 92.2 \%$ & $ 88.5 \%$ & $ 73.1 \%$ &$ 75.6 \%$ & $ 74.6 \%$ & $ 73.1 \%$ & $ 75.1 \%$\\
    tic-tac-toe & $958$ & 24 & 2 & $ 90.8 \%$ & $ 68.5 \%$ & $ 76.1 \%$ & $ 85.7 \%$ & $ 85.8 \%$ & $ 82.1 \%$  &$ 69.6 \%$ & $ 73.6 \%$ & $ 79.6 \%$ & $ 81.0 \%$\\
    \hline
    \end{tabular}
    \label{tab:more-res}
\end{table*}

\begin{table*}
\centering
\footnotesize
\caption{Hyperparameter search spaces for tree-based models. More details about the hyperparamters meaning are given in \cite{komer-proc-scipy-2014}.}
\begin{tabular}{p{2.8cm}p{2.3cm}p{2.3cm}p{2.3cm}p{2.3cm}p{2.3cm}}
\toprule
\textbf{Parameter} & \textbf{CART} & \textbf{Boosted-CART} & \textbf{DPDT} & \textbf{Boosted-DPDT} & \textbf{STreeD} \\
\midrule

\multicolumn{6}{l}{\textit{Common Tree Parameters}} \\
\cmidrule(l){1-6}
max\_depth & \{5: 0.7,\newline 2,3,4: 0.1\} & \{2: 0.4,\newline 3: 0.6\} & \{5: 0.7,\newline 2,3,4: 0.1\} & \{2: 0.4,\newline 3: 0.6\} & 5 \\

min\_samples\_split & \{2: 0.95,\newline 3: 0.05\} & \{2: 0.95,\newline 3: 0.05\} & \{2: 0.95,\newline 3: 0.05\} & \{2: 0.95,\newline 3: 0.05\} & -- \\

min\_impurity\_decrease & \{0.0: 0.85,\newline 0.01,0.02,0.05: 0.05\} & \{0.0: 0.85,\newline 0.01,0.02,0.05: 0.05\} & \{0.0: 0.85,\newline 0.01,0.02,0.05: 0.05\} & \{0.0: 0.85,\newline 0.01,0.02,0.05: 0.05\} & -- \\

min\_samples\_leaf & $\mathcal{Q}(\log\text{-}\mathcal{U}[2,51])$ & $\mathcal{Q}(\log\text{-}\mathcal{U}[2,51])$ & $\mathcal{Q}(\log\text{-}\mathcal{U}[2,51])$ & $\mathcal{Q}(\log\text{-}\mathcal{U}[2,51])$ & $\mathcal{Q}(\log\text{-}\mathcal{U}[2,51])$ \\

min\_weight\_fraction\_leaf   &   \{  0.0: 0.95,\newline 0.01: 0.05\} & \{0.0: 0.95,\newline 0.01: 0.05\} & \{0.0: 0.95,\newline 0.01: 0.05\} & \{0.0: 0.95,\newline 0.01: 0.05\} & -- \\

max\_features & \{"sqrt": 0.5,\newline "log2": 0.25,\newline 10000: 0.25\} & \{"sqrt": 0.5,\newline "log2": 0.25,\newline 10000: 0.25\} & \{"sqrt": 0.5,\newline "log2": 0.25,\newline 10000: 0.25\} & \{"sqrt": 0.5,\newline "log2": 0.25,\newline 10000: 0.25\} & -- \\

\midrule
\multicolumn{6}{l}{\textit{Model-Specific Parameters}} \\
\cmidrule(l){1-6}
max\_leaf\_nodes & \{32: 0.85,\newline 5,10,15: 0.05\} & \{8: 0.85,\newline 5: 0.05,\newline 7: 0.1\} & -- & -- & -- \\

cart\_nodes\_list & -- & -- & 8 configs\newline (uniform) & 5 configs\newline (uniform) & -- \\

learning\_rate & -- & $\log\mathcal{N}(\ln(0.01),\ln(10))$ & -- & $\log\mathcal{N}(\ln(0.01),\ln(10))$ & -- \\

n\_estimators & -- & 1000 & -- & 1000 & -- \\

max\_num\_nodes & -- & -- & -- & -- & \{3,5,7,11,\newline 17,25,31\}\newline (uniform) \\

n\_thresholds & -- & -- & -- & -- & \{5,10,20,50\}\newline (uniform) \\

cost\_complexity & -- & -- & -- & -- & 0 \\

time\_limit & -- & -- & -- & -- & 1800 \\
\bottomrule
\end{tabular}
\label{tab:tree_hyperparams}
\end{table*}
\end{document}